\documentclass{article}
\usepackage[margin=3cm]{geometry}
\usepackage{amsthm}
\usepackage{amsmath}
\usepackage{bm}
\usepackage{algorithm}
\usepackage{algorithmic}
\usepackage[utf8]{inputenc} 
\usepackage[T1]{fontenc}    
\usepackage{hyperref}       
\usepackage{url}            
\usepackage{booktabs}       
\usepackage{amsfonts}       
\usepackage{nicefrac}       
\usepackage{microtype}      

\newcommand{\Var}{\mathbf{Var}}

\def\b{{\bf b}}

\def\E{{\bf E}}
\def\e{{\bf e}}

\def\r{{\bf r}}

\def\u{{\bf u}}
\def\v{{\bf v}}

\def\x{{\bf x}}
\def\y{{\bf y}}
\def\z{{\bf z}}

\def\u{{\bf u}}

\def\v{{\bf v}}

\def\0{{\bf 0}}
\def\1{{\bf 1}}

\def\PB{{\mathbb P}}
\def\QB{{\mathbb Q}}

\def\eps{\epsilon}

\def\tr{\mathrm{tr}}
\def\polylog{\mathrm{ polylog }}

\def\argmin{\mathop{\rm argmin}}

\newtheorem{lemma}{Lemma}
\newtheorem{definition}{Definition}
\newtheorem{theorem}{Theorem}

\newtheorem{proposition}{Proposition}
\newtheorem{cor}{Corollary}
\newtheorem{fact}{Fact}

\usepackage{graphicx}

\title{Sublinear Optimal Policy Value Estimation in Contextual Bandits}
\author{Weihao Kong\\ University of Washington\\ kweihao@gmail.com \and Gregory Valiant\\ Stanford University\\ gvaliant@cs.stanford.edu \and Emma Brunskill\\ Stanford University\\ ebrun@cs.stanford.edu}
\begin{document}

\maketitle

\begin{abstract}
We study the problem of estimating the expected reward of the optimal policy in the stochastic disjoint linear bandit setting. We prove that for certain settings it is possible to obtain an accurate estimate of the optimal policy value even with a number of samples that is sublinear in the number that would be required to \emph{find} a policy that realizes a value close to this optima. 
We establish nearly matching information theoretic lower bounds, showing that our algorithm achieves near optimal estimation error. Finally, we demonstrate the effectiveness of our algorithm on joke recommendation and cancer inhibition dosage selection problems using real datasets. 
\end{abstract}
 
\section{Introduction}

We consider how to efficiently estimate
the best possible performance of the optimal representable
decision policy in a disjoint linear contextual multi-armed bandit setting.
Critically, we are interested in when it is possible
to estimate this best possible performance using a sublinear number of samples, whereas a linear number of samples would typically be required to provide any such policy that can realize optimal performance. 

Contextual multi-armed bandits (see e.g.~\cite{chu2011contextual,li2010contextual,agarwal2014taming}) are a well studied setting
that is having increasing influence and potential impact
in a wide range of applications, including customer
recommendations~\cite{li2010contextual,zhou2016latent}, education~\cite{lan2016contextual} and health~\cite{greenewald2017action}. 
In contrast to simulated domains like games and robotics simulators, in many contextual bandit applications the best potential performance of the 
algorithm is unknown in advance. Such situations will often involve a human-in-the-loop approach to optimizing system performance, where a human expert specifies a set of features describing the potential contexts and a set of possible interventions/arms, and then runs a contextual bandit algorithm to try to identify a high performing decision policy for what intervention to automatically provide in which context. A key challenge facing the human expert is assessing if the current set of context features and set of interventions/arms is yielding sufficient performance. This can be challenging, because without prior knowledge about what optimal performance might be possible, the human may need to run the contextual bandit algorithm until it returns an optimal policy, which may involve wasted time and effort if the best policy representable has mediocre performance. 
%
%
  While there has been some limited algorithmic work on such human-in-the-loop settings for reinforcement learning~\cite{mandel2017add,keramati2019value}  to our knowledge no formal analysis exists of how to efficiently estimate the average reward of the optimal policy representable with the current set of context features and arms.




The majority of prior work on multi-armed bandits has focused on online
algorithms that minimize cumulative or per-step regret (see e.g.\cite{Auer02,agarwal2014taming}).
In simple multi-armed bandit settings (with no context)
there has also been work on maximizing the probability of best arm identification given a fixed 
budget\cite{bubeck2009pure,ABM10,gabillon2012best,karnin2013almost}  or minimizing the number of samples needed to
identify the best arm with high confidence \cite{EMM06,maron1994hoeffding,mnih2008empirical,jamieson2014lil}. Note that in the simple multi-arm bandit setting, sample complexity bounds for $\epsilon$-best arm identification will be equivalent to the bounds achievable for estimating the expected reward of the optimal policy as there is no sharing of rewards or information across arms. 

In the case of contextual multi-armed bandits, there has been some limited work on  single best 
arm identification when the arms are described by a high dimensional feature vector~\cite{hoffman2014correlation,soare2014best,xu2018fully}. However such work does not immediately include input context features (such as from a customer or patient), and would need to be extended to handle best policy identification over (as we consider here) a linear class of policies. 
 A separate literature seeks 
to identify a good policy for future use given access to batch historical data in both bandit and reinforcement learning settings~\cite{thomas2015high,athey2017efficient,gelada2019off,liu2019off}.
In contrast to such work, we consider the setting
where the algorithm may actively gather data, and
the objective is to accurately estimate the performance
of the optimal policy in the set, \emph{without returning a policy that achieves such performance}.
 
In particular, in this work we consider disjoint linear contextual bandits~\cite{li2010contextual} 
(one parameter for each of a finite set of arms, such as a set
of treatments) with
a high dimensional, $d$, input context (such as a set of features
describing the patient). We are interested in providing an
accurate estimate of the expected performance of the best
realizable decision policy. Here the decision policy class is
implicitly defined by the input context feature space
and finite set of arms. Following prior work on disjoint
linear contextual bandits (see e.g.~\cite{li2010contextual}) we assume that the reward for
each arm can be expressed as a linear combination of the
input features and an arm-specific weight vector.

Quite surprisingly, we  present an algorithm that can estimate
the potential expected reward of the best policy
with a number of samples (pulls of the arms) that is
sublinear in the input context dimension $d$. This is
unintuitive because this is less than what is needed to
estimate \textit{any} fit of the $d$-dimensional
arm weight vector, which would require at least $d$ samples.
Our approach builds on recent work~\cite{kong2018estimating} that shows a
related result in the context of regression, showing that
the best accuracy of a regression algorithm can, in
many situations, be estimated with sublinear sample size. 
A critical insight in that paper, which we leverage and
build upon in our work, is the construction of a sequence of unbiased estimators for geometric properties of the data that can be used to estimate the best accuracy, without attempting to find the model achieving that accuracy. However, multiple additional technical subtleties arise when we move from the prediction setting to the control setting because we need to take the interaction between different arms into account while there is effective only one ``arm'' in the prediction setting. Even assuming that we have learned the interaction between the arms, it is not immediately clear how does such knowledge helps determine the potential expected reward of the best policy. We leverages a quantitative version of Sudakov-Fernique inequality to answer the question. While in the classical (non-disjoint) stochastic linear bandit problem, it is crucial to use the information we learned from one arm to infer information for the other arms, this does not hold in the non-disjoint setting. Nevertheless, we utilize the contexts across all the arms to reduce the estimation error, which yields a near optimal sample complexity dependency on the number of arms.

Our key contribution is an algorithm for accurately
estimating the expected performance of the optimal
policy in a disjoint contextual linear bandit setting
with an amount of samples that is sublinear in the
input context dimension. We provide theoretical bounds 
when the input context distributions are  drawn from Gaussians with zero mean and known or unknown covariances. We then examine the performance empirically, first in a synthetic setting. We then evaluate our method both in identifying the optimal reward for a joke recommendation decision policy, based on the Jester dataset~\cite{goldberg2001eigentaste}, and on a new task we introduce of predicting the performance of the best linear threshold policy for selecting the dosage level to optimize cancer cell growth inhibition in the NCI-60 Cancer
Growth Inhibition dataset. Encouragingly, our results suggest that our algorithm quickly obtains an accurate estimate of the optimal linear policy.



\section{Problem Setting}

A contextual multi-armed bandit (CMAB) can be described
by a set of contexts $\mathcal{X} \in R^d$, a set of $K$ arms $\mathcal{K}$ and
a reward function. We consider  the linear disjoint CMAB setting~\cite{li2010contextual},
where there are a finite set of arms, and the reward $y$ from
pulling an arm $a$ in a context $\x_j$ is
\begin{equation}
  y_{a,j}=\beta_a^T\x_j+b_a+\eta_{a,j}.
\end{equation}
For each arm $a$, $\beta_a$ is an unknown $d$-dimensional real vector with bounded $\ell_2$ norm and $b_a$ is a real number. 
$\E[\eta]=0, \E[\eta^2]\le \sigma^2$ where $\sigma$ is a constant. 

For simplicity, we focus primarily on the passive setting where 
for each arm $a$, we observe $N$ iid samples $\x_{a,1},\x_{a,2},\ldots,\x_{a,N}$ drawn from $N(0,\Sigma)$, and each sample $\x_{a,j}$ is associated with a reward. Under this setting, we denote $\sigma_a^2$ as the variance of $y_{a,j}$, which is smaller than $\beta_a^\top\Sigma\beta_a+\sigma^2$, and it is assumed that $\sigma_a$ are all bounded by a constant. We define the total number of samples $T=K\cdot N$ to draw a connection to the adaptive setting where the algorithm can adaptive choose the action to play on each context. Interestingly, in the worst case our approach of uniformly gathering samples across all actions is optimal up to a $\log^{3/2} (dK)$ factor (see Theorem~\ref{thm:lb-adaptive}).

Given a total of $T = K\cdot N$ samples $(\x_{a,j},y_j)$, our goal is to predict the expected reward of the optimal policy realizable with the input
definition of context features and finite set of actions, which is $OPT:= \E_{\x}[\max_a(\beta_a^T\x+b_a)]$. 

\section{Summary of Results}
 Our first result applies to the setting where each context is drawn from a $d$-dimensional  Gaussian distribution $N(0,\Sigma)$, with a \emph{known} covariance matrix $\Sigma$, and the reward for the $a$th arm on context $\x$ equals $\beta_a^T\x+b_a+\eta_a$ where $\E[\eta_a]=0, \E[\eta_a^2]$ is bounded by a constant.\footnote{The setting where the covariance, $\Sigma$, is known is equivalent to the setting where the covariance is assumed to be the identity, as the data can be re-projected so as to have identity covariance.  While the assumption that the covariance is known may seem stringent, it applies to the many settings where there is a large amount of \emph{unlabeled} data.  For example, in many medical or consumer data settings, an accurate estimate of the covariance of $\x$ can be obtained from large existing databases.}  Given $N=\Theta(\eps^{-2}\sqrt{d}\log K \log (K/\delta))$ samples for each arm, there is an efficient algorithm that with probability $1-\delta$ estimates the optimal expected reward with additive error $\eps$.  
\begin{cor}[Main result, known covariance setting]\label{cor:isotropic}
In the known covariance setting, for any $\eps\ge \frac{\sqrt{\log K}}{d^{1/4}}$, with probability $1-\delta$, Algorithm~\ref{alg:main} estimates the optimal reward $OPT$ with additive error $\eps$ using a total number of samples 
$$
T = \Theta(\frac{\sqrt{d}K\log K}{\eps^2} \log (K/\delta)).
$$
\end{cor}

We prove a near matching lower bound, showing that in this passive setting, the estimation error can not be improved by more than a $\log K$ factor. The proof of Theorem~\ref{thm:lb} can be found in the supplementary material. 

\begin{theorem}[Lower bound for passive algorithms, known covariance setting]\label{thm:lb}
There exists a constant $C$ such that for any $\eps>0$ given 
$$
T = C\frac{\sqrt{d}K\log K}{\eps^2}
$$
   samples (equivalently $N = C\frac{\sqrt{d}\log K}{\eps^2}$ samples for each arm), no algorithm can estimate the optimal reward with expected additive error less than $\eps$ with probability greater than $2/3$.
\end{theorem}
Comparing against the adaptive setting where the algorithm can adaptively choose the action to play on each context, we prove a surprising lower bound, showing that the estimation error can not be improved much. Specifically, our passive algorithm is minimax optimal even in the \textit{adaptive setting} up to a $\polylog(dK)$ factor. The proof is deferred to the supplementary material.

\begin{theorem}[Lower bound for fully adaptive algorithms, known covariance setting]\label{thm:lb-adaptive}
There exists a constant $C$ such that no algorithm can estimate the optimal reward with additive error $\eps$ and probability  of success at least $2/3$ using a number of rounds that is less than 
$$T = C\frac{\sqrt{d}K}{\eps^{2}\log^{3/2}(dK)}.$$
\end{theorem}
Our lower bound is novel, and we are not aware of similar results in this setting. It is curious that the standard approach by simply bounding the KL-divergence only yields a sub-optimal $\tilde{O}(\sqrt{dK})$ lower bound, since the divergence contribution of each arm scales with $\E[T_i^2]$ instead of $\E[T_i]$ in the classical (non-contextual) stochastic bandit setting. We apply a special conditioning to get around this issue. 

Our algorithmic techniques apply beyond the isotropic covariance setting, and we prove an analog of Corollary~\ref{cor:isotropic}  in the setting where the contexts $\x$ are drawn from a Gaussian distribution with arbitrary \emph{unknown} covariance. Our general result, given in Corollary~\ref{cor:unknown}, is quite complicated. Here we highlight the special case where the desired accuracy $\eps$ and failure probability $\delta$ is positive constant, and the covariance is well-conditioned:

\begin{cor}[Special case of main result, unknown covariance setting]\label{cor:main-gen}
Assuming that the covariance of the context $\x_i$ satisfies $\sigma_{\min} I_d \preceq \Sigma\preceq \sigma_{\max} I$ and $\sigma_{\max}/\sigma_{\min}$ is a constant, for constant $\eps$, Algorithm 1 takes $\sigma_{\min}, \sigma_{\max}$, and a total number of 
$$
T = O(d^{1-\frac{C}{\log\log K+\log(1/\eps)}}K^{\gamma}+\sqrt{d}K^{1+\gamma})
$$
samples, where $\gamma$ is any positive constant and $C$ is a universal constant. 
\end{cor}
In the unknown covariance setting, the dependency on $d$ of our algorithm is still sublinear, though is much worse than the $\sqrt{d}$ dependency in the known covariance setting. However this can not be improved by much as the lower bound result in~\cite{kong2018estimating} implies that the dependency on $d$ is at least $d^{1-\Theta(\frac{1}{\log 1/\eps})}$

It is worth noting that the techniques behind our result in the unknown covariance setting that achieves sublinear sample complexity essentially utilizes a set of unlabeled examples of size $O(T)$ to reduce the variance of the estimator, where unlabeled examples are the contexts vectors $\x$'s drawn from $N(0,\Sigma)$. If one has an even larger set of unlabeled examples, the samples complexity for the labeled examples can be significantly reduced. For simplicity, we do not present a complete trade-off result between the labeled and unlabeled examples in this paper. Instead, we present one extreme case where there is a sufficiently large set of unlabeled examples (size $\Omega(d)$), and the problem essentially becomes a known covariance problem. 

\begin{cor}[Unknown covariance with a large set of unlabeled examples]\label{cor:gen-unlabeled}
In the unknown covariance setting, there is an algorithm that estimates the optimal reward OPT with additive error $\eps$ with probability $1-\delta$ using a total number of labeled examples
$$
T = \Theta(\frac{\sqrt{d}K\log K}{\eps^2} \log (K/\delta)),
$$
and a set of unlabeled examples of size $\Theta((d+\log 1/\delta)\log^2 K/\eps^4)$ 
\end{cor}

The algorithm that achieves the above result is very straight forward. We first estimate the covariance of the context using the set of unlabeled examples up to $\eps$ spectral norm error, and let us denote $\hat{\Sigma}$ as the estimator. Given the covariance estimator, we will execute the known covariance version Algorithm~\ref{alg:main} and scale each context $\x_i$ as $\hat{\Sigma}^{-1/2}\x_i$. The covariance of the scaled context is not exactly identity, hence our estimator is biased. However, it is straight forward to show that the bias is at most $O(\eps)$, which is on the same magnitude of the standard deviation of our estimator. The proof is deferred to the appendix.

Finally, we made an attempt to generalize our result beyond the Gaussian context setting and showed that if each context is drawn from a mixture of $M$ Gaussians distribution which is completely known to the algorithm, then our algorithm can be applied to achieve $\eps$ estimation error while the sample complexity only increases by a factor of $\log M$. The proof is deferred to the appendix. 
\begin{theorem}[Extension to the mixture of Gaussians setting]\label{thm:main-mog}
Suppose each context $\x$ is drawn independently from a mixture of Gaussian distribution $\sum_{i=1}^M \alpha_i N(\mu_i,\Sigma_i)$, and the parameters $\mu_i, \Sigma_i, \alpha_i$ are all known to the algorithm. In addition, let us assume that $\|\mu_i\|$, $\|\Sigma_i\|$ are all bounded by a constant. Then for any $\eps\ge \frac{\sqrt{\log K}}{d^{1/4}}$, with probability $1-\delta$, there is an algorithm that estimate the optimal reward $OPT$ with additive error $\eps$ using a total number of samples 
$$
T = \Theta(\frac{\sqrt{d}K\log K}{\eps^2}\log (KM/\delta)).
$$
\end{theorem}


\section{The Estimators}
The basic idea of our estimator for the optimal reward of linear contextual bandits is as follows. For illustration, we assume that each context $\x$ is drawn from a standard Gaussian distribution $N(0,I_d)$. In the realizable setting where the reward for pulling arm $a$ on context $\x$ is $\beta_a^T\x+b_a+\eta$ where $\beta_a$, $b_a$ are the parameters associated with arm $a$ and $\eta$ is random noise with mean $0$, the expected reward of the optimal policy is simply $\E_\x[\max_a(\beta_a^T\x+b_a)]$. Let us define the $K$ dimensional random variable $r = (\beta_1^T\x+b_1,\beta_2^T\x+b_2,\ldots,\beta_K^T\x+b_K)$. Notice that in the setting where $\x\sim N(0,I)$, $r$ is an $K$ dimensional Gaussian random variable with mean $\b=(b_1,\ldots,b_K)$ and covariance $H$ where $H_{a,a'} = \beta_a^T\beta_{a'}$. Hence in this simplified setting, the optimal reward of the linear contextual bandit problem can be expressed as $\E_{\r\sim N(\b,H)}[\max_i r_i]$ which is a function of $\b$ and $H$. Naturally, one can hope to estimate the optimal reward by first accurately estimating $\b$ and $H$. The bias $\b$ can be accurately estimated up to entry-wise error $O(\sqrt{\frac{1}{N}})$ by computing the average of the reward of each arm, simply because for any $i$, $y_{a,i}$ is an unbiased estimator of $b_a$. 

Very recently, the authors of~\cite{kong2018estimating} proposed an estimator for $\beta^T\beta$ in the context of learnability estimation, or noise level estimation for linear regression. In the setting where each covariate $\x_i$ is drawn from a distribution with zero mean and identity covariance, and response variable $y_i=\beta^T\x_i+\eta_i$ with independent noise $\eta_i$ having zero mean, they observe that for any $i\ne j$, $y_iy_j\x_i^T\x_j$ is an unbiased estimator of $\beta^T\beta$. In addition, they showed that the error rate of estimating $\beta^T\beta$ using the proposed estimator $\frac{1}{\binom{N}{2}}\sum_{i\ne j}y_iy_j\x_i^T\x_j$ is $O(\frac{d+N}{N^2})$ which implies that one can accurately estimate $\beta^T\beta$ using $N = O(\sqrt{d})$ samples. Their estimator can be directly applied to estimate $\beta_a^T\beta_a$, and we extend their techniques to the contextual bandit setting for estimating $\beta_a\beta_{a'}$ for arbitrary $a,a'$. In order to estimate $\beta_a^T\beta_{a'}$ for $a\ne a'$, notice that for any $i,j$, $\E[\y_{a,i}\y_{a',j}\x_{a,i}^T\x_{a',j}] = \beta_a^T\x_{a,i}\x_{a,i}^T\x_{a',j}\x_{a',j}^T\beta_{a'} = \beta_a^T\beta_{a'}$, and we simply take the average of all all these unbiased estimators of $\beta_a^T\beta_{a'}$. We show that $O(\sqrt{d})$ samples for each arms suffices for accurate estimation of $\beta_a^T\beta_{a'}$ for arbitrary pairs of arms $a,a'$.

Once we have estimates $\hat{\b}, \hat{H}$ for $\b$ and $H$, if $\hat{H}$ is a PSD matrix, our algorithm simply outputs $\E_{\r\sim N(\hat{\b},\hat{H})}(\max_i r_i)$, otherwise, let $\hat{H}^{(PSD)}$ be the projection of $\hat{H}$ on to the PSD cone and output $E_{\r\sim N(\hat{\b},\hat{H}^{(PSD)})}(\max_i r_i)$. Given an approximation of $\b, H$, it is not immediately clear how do the errors on estimating $\b, H$ translate to the error on estimating the $\E_{\r\sim N(\b,H)}[\max_i r_i]$. Our Proposition~\ref{prop:bH2OPT} leverages a quantitative version of Sudakov-Fernique inequality due to Chatterjee~\cite{chatterjee2005error} and shows that if each entry of $H$ is perturbed by at most $\eps$, the optimal reward $\E_{\r\sim N(\b,H)}[\max_i r_i]$ can only change by $2\sqrt{\log K \eps}$. Because $\E_{\x\sim N(\hat{\b},\hat{H})}(\max_i x_i+b_i)$ has no closed-form expression in general, we use Monte Carlo simulation to approximate $\E_{\x\sim N(\hat{\b},\hat{H})}(\max_i x_i+b_i)$ in the implementation. 

Our estimator for the general unknown covariance setting is much more involved. Assuming each context $\x$ is drawn from an Gaussian distribution with zero mean and unknown covariance $\Sigma$, the optimal reward $\E_{\x\sim N(0,\Sigma)}[\max_a(\beta_a^T\x+b_a)]$ is equal to $\E_{\r\sim N(\b,H)}[\max_i r_i]$ where $\r\sim N(\b, H)$ and $H_{a,a'} = \beta_a\Sigma\beta_{a'}$. Again, we extend the estimator proposed in~\cite{kong2018estimating} for $\beta^T\Sigma\beta$
in the linear regression setting to the contextual linear bandit setting for estimating $\beta_a\Sigma\beta_{a'}$ for arbitrary $a,a'$. For each $a,a'$, we design a series of unbiased estimators for $\beta_a^T\Sigma^2\beta_{a'}, \beta_a^T\Sigma^3\beta_{a'}, \beta_a^T\Sigma^4\beta_{a'},\ldots$ and approximate $\beta_a^T\Sigma\beta_{a'}$ with a linear combination of these high order estimates. Our major contribution is a series of estimators which incorporate unlabeled examples. In the contextual bandit setting, especially when $K$ is large, it is essential to incorporate unlabeled data, simply because when we estimate $\beta_a\Sigma^k\beta_{a'}$, the large number of examples which do not involve arm $a$ or $a'$ are effectively unlabeled examples and can be leveraged to significantly reduce the overall variance for estimating $\beta_a\Sigma^k\beta_{a'}$. We prove variance bounds in Corollary~\ref{cor:Hgen} for these novel estimators whose accuracy depends on both the number of labeled examples and unlabeled examples. As a side note, our estimator can also be applied to the setting of estimating learnability to better utilize the unlabeled examples. Proofs, where omitted, are in the appendix. 

\subsection{Main Algorithm}
Our main algorithm is described as Algorithm~\ref{alg:main}.
In line~\ref{alg:line:repeat}, we repeat the for loop body $\Theta(\log (K/\delta))$ times, and at each time, we collect $n$ i.i.d. sample for each arm. Hence the total number of samples for each arm $N=\Theta( n(\log (K/\delta))$. For ease of notations, we will use $n$ instead of $N$ when we write down the error rate of the algorithm.

In line~\ref{alg:line:collect1}, \ref{alg:line:collect2}, \ref{alg:line:collect3}, for each arm $a$ we collect $n$ i.i.d. samples and estimate the bias of that arm $b_a$. The estimation error of the bias vector $\b$ is bounded by the following corollary, and the claim holds by applying Chebyshev's inequality with the variance of $y_{a,i}$.
\begin{cor}\label{cor:var-b}
For each arm $a$, with probability $2/3$, $|\frac{1}{n}\sum_{i=1}^n y_{a,i}-b_a|\le 3\sqrt{\frac{1}{n}}\sigma_a$, where $\sigma_a^2 = \Var[y_{a,i}]$.
\end{cor}
After estimating $b_a$, we can subtract $b_a$ off from all the $y_{a,i}$. For sufficiently large $n$, our estimation of $b_a$ is accurate enough such that we can assume that $y_{a,i}=\beta_a^T\x_{a,i}+\eta_{a,i}$. After collecting $n$ i.i.d. samples from each arm, in the known covariance setting, we run Algorithm~\ref{alg:Hiso} to estimate the covariance $H$ in line~\ref{alg:line:iso-est-H}. In the unknown covariance setting, we need to split the $n$ examples for each arm into one labeled example set and one unlabeled examples set, and then run Algorithm~\ref{alg:Hgen} to estimate the covariance $H$. Bounds on Algorithm~\ref{alg:Hiso} and Algorithm~\ref{alg:Hgen} are formulated in the following two corollaries.
\begin{cor}\label{cor:Hiso}
Given $n$ independent samples for each arm, for a fixed pair $a,a'$, with probability at least $2/3$, the output of Algorithm~\ref{alg:Hiso} satisfies
$$
|\hat{H}_{a,a'}-H_{a,a'}| \le 3\sqrt{\frac{9d+3n}{n^2}}\sigma_a\sigma_{a'},
$$
where $\sigma_a^2 = \Var[y_{a,i}]$.
\end{cor}
The above corollary follows from applying Chebyshev's inequality with the variance bound established in Proposition~\ref{prop:Hkkp-iso} and Proposition~\ref{prop:Hkk-iso}.
\begin{cor}\label{cor:Hgen}
Given $n$ independent samples for each arm, and $s$ unlabeled examples, for a fixed pair $a,a'$, with probability at least $2/3$, the output of Algorithm~\ref{alg:Hgen} satisfies
\begin{align*}
|\hat{H}_{a,a'}-H_{a,a'}| = \min(\frac{2}{k^2},2e^{-(k-1)\sqrt{\frac{{\sigma_{min}}}{{\sigma_{max}}}}})\\
+f(k)\max(\frac{d^{k/2}}{s^{k/2}},1)\sqrt{\frac{d+n}{n^2}},
\end{align*}
where $f(k) = k^{O(k)}$.
\end{cor}
The above corollary follows from applying Chebyshev's inequality with the variance bound established in Proposition~\ref{prop:poly}. Notice that after sample splitting in Algorithm~\ref{alg:main}, the size of the set of the unlabeled examples $s=Kn/2$. 

Since each entry of our estimation of $\b$, the output of Algorithm~\ref{alg:Hiso} and Algorithm~\ref{alg:Hgen} only satisfies the bound in Corollary~\ref{cor:var-b}, Corollary~\ref{cor:Hiso} or Corollary~\ref{cor:Hgen} respectively with probability $2/3$, we boost the entry-wise success probability to $1-\delta/(K^2+K)$ by repeating the estimation procedure $\Theta(\log(K/\delta))$ times and compute the median of our estimation (line~\ref{alg:line:median-H} to line~\ref{alg:line:median-b}), such that the overall success probability is at least $1-\delta$. We formalize the effect of this standard boosting procedure in Fact~\ref{fact:boosting}.

Line~\ref{alg:line:proj} projects the matrix $\hat{H}$ onto the PSD cone and obtains the PSD matrix $\hat{H}_{PSD}$. This step is a convex optimization problem and can be solved efficiently. By the triangle inequality and the upper bound of $\max_{i,j} |\hat{H}_{i,j}-H_{i,j}|$, the discrepancy after this projection: $\max_{i,j}|\hat{H}^{(PSD)}_{i,j}-H_{i,j}|$ can be bounded with the upper bound in Corollary~\ref{cor:Hiso} and Corollary~\ref{cor:Hgen} up to a factor of $2$.

Now that we have established upper bounds on the estimation error of $\b$ and $H$, we use these to bound 
the estimation error of the optimal reward.
\begin{proposition}\label{prop:bH2OPT}
Let $H\in R^{m\times m}$ and $H'\in R^{m\times m}$ be two PSD matrices, $\b,\b'$ be two $d$-dimensional real vectors. We have $|\E_{\x\sim N(\b,H)}[\max_i x_i]-\E_{\x\sim N(\b',H')}[\max_i x_i]|\le 2\sqrt{\max_{i,j}|H_{i,j}-H'_{i,j}|\log K}+\max_i|b_i-b_i'|$.
\end{proposition}

\begin{algorithm}[ht!]
\begin{algorithmic}[1]
\FOR{$i=1$ \TO $\lceil 48(\log (K^2/\delta)+1)\rceil$}\label{alg:line:repeat}
    \FOR{$a=1$ \TO $K$}
    \STATE Pull the $a$'th arm $n$ times, and let matrix $X_a =
    \begin{bmatrix} 
    \x_{a,1}^\top \cdots \x_{a,n}^\top
    \end{bmatrix}^\top$ consists of the $n$ contexts, $\y_a = 
    \begin{bmatrix} 
    y_{a,1} \cdots y_{a,n}
    \end{bmatrix}^\top$ consists of the $n$ rewards.\label{alg:line:collect1}
    \STATE $\hat{b}_{a}^{(i)} \gets \bm{1}^T\y_a/n.$\label{alg:line:collect2}    \COMMENT{Estimate $b_a$.}
    \STATE $\y_a \gets \y_a-\hat{b}_a^{(i)}\bm{1}$.\label{alg:line:collect3} \COMMENT{Subtract $b_a$ off to make it zero mean.}
    \ENDFOR
    \IF{Known covariance} 
    \STATE $\hat{H}^{(i)}\gets$ \textbf{Algorithm~\ref{alg:Hiso}}$(\{X_a\}_{a=1}^K, \{\y_a\}_{a=1}^K)$.\label{alg:line:iso-est-H} \COMMENT{Corollary~\ref{cor:Hiso}}
    \ELSE
    \FOR{$a=1$ \TO $K$}
        \STATE $X_a\gets \begin{bmatrix} 
        \x_{a,1}^\top \cdots \x_{a,n/2}^\top
        \end{bmatrix}^\top$. 
        \STATE $\y_a\gets \begin{bmatrix} 
        y_{a,1}^\top \cdots y_{a,n/2}^\top
        \end{bmatrix}^\top$.
        \STATE $S_a \gets \begin{bmatrix} 
        \x_{a,n/2+1}^\top \cdots \x_{a,n}^\top
        \end{bmatrix}^\top$. \COMMENT{Split $\x$ into a labeled example set and an unlabeled example set}
    \ENDFOR
    \STATE $S\gets \begin{bmatrix} 
        S_1^\top \cdots S_K^\top
        \end{bmatrix}^\top$.
    \STATE{$\hat{H}^{(i)}\gets$ \textbf{Algorithm}~\ref{alg:Hgen}$ ( \{X_a\}_{i=a}^K, \{\y_a\}_{a=1}^K, S, p(x) $).} \COMMENT{Corollary~\ref{cor:Hgen}}
    \ENDIF\COMMENT{Estimate $H$.}
\ENDFOR
\STATE For all $1 \le i,j\le K$,\\
$\hat{H}_{i,j} \gets \textbf{median}(\hat{H}_{i,j}^{(1)},\ldots, \hat{H}_{i,j}^{(\lceil 48(\log (K^2/\delta)+1)\rceil)})$.\label{alg:line:median-H}
\STATE For all $1\le i\le K$,\\
$\hat{b}_i\gets \textbf{median}(\hat{b}_i^{(1)},\ldots, \hat{b}_i^{(\lceil 48(\log (K^2/\delta)+1)\rceil)})$.\label{alg:line:median-b}
\STATE $\hat{H}^{(PSD)} \gets \argmin_{M\succ 0}\max_{i,j}|\hat{H}_{i,j}-M_{i,j}|$ \COMMENT{Project onto the PSD cone under the max norm.}\label{alg:line:proj}
\STATE \textbf{Output}:$\E_{\r\sim N(\hat{\b},\hat{H}^{(PSD)})}[\max_i r_i]$. 
\caption{Main Algorithm for Estimating  $OPT$, the Optimal Reward [Corollary~\ref{cor:isotropic}, Corollary~\ref{cor:main-gen}]}\label{alg:main}
\end{algorithmic}
\end{algorithm}

We are ready to state our main theorem for the known covariance setting.
\begin{theorem}[Main theorem on Algorithm~\ref{alg:main}, known covariance setting]\label{thm:iso}
In the known covariance setting, with probability at least $1-\delta$, Algorithm~\ref{alg:main} estimates the expected reward of the optimal policy with error bounded as follows:
$$
|OPT-\widehat{OPT}| = O(\sqrt{\log K}(\frac{d+n}{n^2})^{1/4})
$$
\end{theorem}

For the following main theorem on the general unknown covariance setting, the proof is identical to the proof of Theorem~\ref{thm:iso}. 
\begin{theorem}[Main theorem on Algorithm~\ref{alg:main}, unknown covariance setting]
\label{thm:unknownc}
In the unknown covariance setting, for any positive integer $k$, with probability $1-\delta$, Algorithm~\ref{alg:main} estimates the optimal reward OPT with additive error:
\begin{align*}
&|OPT-\widehat{OPT}| \le O\bigg(\sqrt{\log K} \big(\min(\frac{1}{k^2},e^{-(k-1)\sqrt{\frac{{\sigma_{min}}}{{\sigma_{max}}}}})\\
&+f(k)\max(\frac{d^{k/2}}{s^{k/2}},1)\sqrt{\frac{d+n}{n^2}}\Big)^{1/2}\bigg),
\end{align*}
where $f(k) = k^{O(k)}$.
\end{theorem}
Choosing the optimal $k$ in Theorem~\ref{thm:unknownc} yields the following Corollary~\ref{cor:unknown} on the overall sample complexity in the unknown covariance setting.
\begin{cor}\label{cor:unknown}
For any $\eps>\frac{\sqrt{\log K}}{d^{1/4}}$, with probability $1-\delta$, Algorithm 1 estimates the optimal reward OPT with additive error $\eps$ using a total number of 
\begin{align*}
T =& \Theta\Big(\log(K/\delta)\max(k^{O(1)}d^{1-1/k}K^{2/k}, \frac{k^{O(k)}K \log K\sqrt{d}}{\eps^2})\Big)
\end{align*}
samples, where $k = \min(C_1\sqrt{\log K}/\eps+2, \sqrt{\frac{\sigma_{\max}}{\sigma_{\min}}}(\log(\log K/\eps^2)+C_2))$ for universal constants $C_1, C_2$.
\end{cor}

In the next two sections, we describe our estimators for $H$ in both known and unknown covariance settings.

\subsection{Estimating $H$ in the Known Covariance Setting}
In this section, we show that the output of Algorithm~\ref{alg:Hiso} satisfies Proposition~\ref{prop:Hkk-iso} and Proposition~\ref{prop:Hkkp-iso}. As stated earlier, we assume $\Sigma=I$ and $\E[\x]=0$ in this section.


\begin{algorithm}[ht]
\begin{algorithmic}[1]
\STATE \textbf{Input}: $X_1 = 
\begin{bmatrix} 
\x_{1,1}\\
\vdots\\
\x_{1,n}
\end{bmatrix},\ldots, X_K = 
\begin{bmatrix} 
\x_{K,1}\\
\vdots\\
\x_{K,n}
\end{bmatrix},\quad 
\y_1 = 
\begin{bmatrix} 
y_{1,1}\\
\vdots\\
y_{1,n}
\end{bmatrix},\ldots,\y_K = 
\begin{bmatrix} 
y_{K,1}\\
\vdots\\
y_{K,n}
\end{bmatrix}$
\FOR{$a=1$ \TO $K$}
\STATE $A \gets (X_aX_a^T)_{up}$ where $(X_aX_a^T)_{up}$ is the matrix $X_aX_a^T$ with the diagonal and lower triangular entries set to zero.
\STATE $\hat{H}_{a,a} \gets \y_a^TA_{up}\y_a/\binom{n}{2}$.
\FOR{$a'=a+1$ \TO $K$}
\STATE $\hat{H}_{a,a'} \gets \y_a^TX_aX_{a'}^T\y_{a'}/\binom{n}{2}$.
\STATE $\hat{H}_{a',a} \gets \hat{H}_{a,a'}$.
\ENDFOR
\ENDFOR
\STATE \textbf{Output}: $\hat{H}$.
\caption{Estimating $\beta_a^T\beta_{a'}$, Identity covariance [Proposition~\ref{prop:Hkk-iso}, Proposition~\ref{prop:Hkkp-iso}]}\label{alg:Hiso}
\end{algorithmic}
\end{algorithm}




To bound the estimation error of $H$, first observe that  $\hat{H}_{a,a}\frac{\y_a^TA_{up}\y_a}{\binom{n}{2}}$ computed in Algorithm~\ref{alg:Hiso} is equal to $\frac{1}{\binom{n}{2}}\sum_{i<j} y_{a,i}y_{a,j}\x_{a,i}^T\x_{a,j}$. The following proposition on the estimation error of $\hat{H}_{a,a}$ is a restatement of Proposition 4 in~\cite{kong2018estimating}.
\begin{proposition}[Restatement of Proposition 4 in~\cite{kong2018estimating}]\label{prop:Hkk-iso}
For each arm $a$, define $\hat{H}_{a,a} = \frac{1}{\binom{n}{2}}\sum_{i<j} y_{a,i}y_{a,j}\x_{a,i}^T\x_{a,j}$ and $H_{a,a} = \beta_a^T\beta_a$. Then $\E[\hat{H}_{a,a}] = H_{a,a}$ and
$
\E[(\hat{H}_{a,a}-H_{a,a})^2] \le \frac{9d+3n}{n^2}\sigma_a^4
$.
\end{proposition}
The estimate $\hat{H}_{a,a'} = \y_a^TX_aX_{a'}^T\y_{a'}/\binom{n}{2}$ computed in Algorithm~\ref{alg:Hiso} is equivalent to $\frac{1}{n^2}\sum_{i,j} y_{a,i}y_{a',j}\x_{a,i}^T\x_{a',j}$, and the following proposition bounds the estimation error of $\hat{H}_{a,a'}$ for $a\ne a'$.
\begin{proposition}\label{prop:Hkkp-iso}
For a pair of arms $a,a'$, define $\hat{H}_{a,a'} = \frac{1}{n^2}\sum_{i,j} y_{a,i}y_{a',j}\x_{a,i}^T\x_{a',j}$ and $H_{a,a'} = \beta_a^T\beta_{a'}$. Then $\E[\hat{H}_{a,a'}] = H_{a,a'}$ and $\E[(\hat{H}_{a,a'}-H_{a,a'})^2]\le \frac{9d+3n}{n^2}\sigma_a^2\sigma_{a'}^2$.
\end{proposition}
\begin{proof}
We need the following fact about the $4$-th moment of Gaussian distribution in the proof of this proposition.
\begin{fact}\label{fact:gaussian4mom}
Let $\x\sim N(0,I_d)$. $\E[(\u^T\x)^2(\v^T\x)^2] = \|\u\|^2\|\v\|^2+2(\u^T\v)^2$
\end{fact}

It's easy to verify that $\E[\hat{H}_{a,a'}] = H_{a,a'}$. $\E[(\hat{H}_{a,a'}-H_{a,a'})^2]$ can be expressed as \begin{align*}
&\frac{1}{n^4}\sum_{i,j,i',j'}(\E[y_{a,i}y_{a',j}y_{a,i'}y_{a',j'}\x_{a,i}^T\x_{a',j}\x_{a,i'}^T\x_{a',j'}]\\
&-\E[y_{a,i}y_{a',j}\x_{a,i}^T\x_{a',j}]\E[y_{a,i'}y_{a',j'}\x_{a,i'}^T\x_{a',j'}]).\end{align*}
For each term in the summation, we classify it into one of the $3$ different cases according to $i,j,i',j'$:
\begin{enumerate}
\item If $i\ne i'$ and $j\ne j'$, the term is $0$.
\item If $i=i'$ and $j\ne j'$, the term can then be expressed as:
$\E[y_{a,i}^2y_{a',j}y_{a',j'}\x_{a,i}^T \x_{a',j} \x_{a,i}^T \x_{a',j'}] - (\beta_a^T\beta_{a'})^2
= \E[y_{a,i}^2(\beta_{a'}^T \x_{a,i})^2]- (\beta_{a}^T\beta_{a'})^2 \le \E[(\beta_a\x_{a,i})^2(\beta_{a'}\x_{a,i})^2]+\sigma^2\|\beta_{a'}\|^2\le 3\sigma_{a}^2\sigma_{a'}^2$. The last equality follows from Fact~\ref{fact:gaussian4mom}.
\item If $i\ne i'$ and $j=j'$, this case is symmetric to the last case and $3\sigma_{a}^2\sigma_{a'}^2$ is an upper bound.
\item If $i=i'$ and $j=j'$, the term can then be expressed as:
$\E[y_{a,i}^2y_{a',j}^2(\x_{a,i}^T \x_{a',j})^2] - (\beta_a^T\beta_{a'})^2$. First taking the expectation over $\x_{a',j},y_{a',j}$, we get the following upper bound 
$3\E[y_{a,i}^2(\x_{a,i}^T\x_{a,i})]\sigma_{a'}^2$. Notice that $\x_{a,i}^T\x_{a,i}=\sum_{l=1}^d(\e_j^T\x_{a,i})^2$. Taking the expectation over the $i$th sample and applying the fourth moment condition of $\x$, we get the following bound:
$
9d\sigma_a^2\sigma_{a'}^2.
$
\end{enumerate}
The final step is to sum the contributions of these $3$ cases. Case $2$ and $3$ have $4\binom{n}{3}$ different quadruples $(i,j,i',j')$. Case $4$ has $\binom{n}{2}$ different quadruples $(i,j,i',j')$. Combining the resulting bounds yields:
$\frac{1}{n^4}\sum_{i,j,i',j'}(\E[y_{a,i}y_{a',j}y_{a,i'}y_{a',j'}\x_{a,i}^T\x_{a',j}\x_{a,i'}^T\x_{a',j'}]-\E[y_{a,i}y_{a',j}\x_{a,i}^T\x_{a',j}]\E[y_{a,i'}y_{a',j'}\x_{a,i'}^T\x_{a',j'}])\le \frac{3n+9d}{n^2}\sigma_a^2\sigma_{a'}^2.
$
\end{proof}

\subsection{Estimating $H$ in the Unknown Covariance Setting}
In this section, we present the algorithm for estimating $H$ in the unknown covariance setting and its main proposition. We assume each context $\x_{a,i}$ of the input of Algorithm~\ref{alg:Hgen} is drawn from $N(0,\Sigma)$.
\begin{algorithm}[ht]
\begin{algorithmic}[1]
\STATE \textbf{Input}: $X_1 = 
\begin{bmatrix} 
\x_{1,1}\\
\vdots\\
\x_{1,n}
\end{bmatrix},\ldots, X_K = 
\begin{bmatrix} 
\x_{K,1}\\
\vdots\\
\x_{K,n}
\end{bmatrix},$ 
$\y_1 = 
\begin{bmatrix} 
y_{1,1}\\
\vdots\\
y_{1,n}
\end{bmatrix},\ldots,\y_K = 
\begin{bmatrix} 
y_{K,1}\\
\vdots\\
y_{K,n}
\end{bmatrix}$, unlabeled examples $X = \begin{bmatrix} 
\x_{1}\\
\vdots\\
\x_{s}
\end{bmatrix}$ and degree $k+2$ polynomial $p(x)=\sum_{i=0}^{k} a_i x^{i+2}$ that approximates the function $f(x)=x$ for all $x\in [\sigma_{\min},\sigma_{\max}],$ where $\sigma_{\min}$ and $\sigma_{\max}$ are the minimum and maximum singular values of the covariance of the distribution from which the $\x_i$'s are drawn.
\STATE $G \gets (X X^T)_{up}$ where $(X X^T)_{up}$ is the matrix $X X^T$ with the diagonal and lower triangular entries set to zero.
\STATE $P \gets a_0I_d + \sum_{t=1}^k \frac{a_t}{\binom{s}{t}}X^TG^{t-1}X$.
\FOR{$i=1$ \TO $m$}
\STATE $\hat{H}_{a,a} \gets \y_a^T(X_aPX_a^T)_{up}\y_a / \binom{n}{2}$.
\FOR{$a'=a+1$ \TO $m$}
     \STATE $\hat{H}_{a,a'} \gets \y_a^TX_aPX_{a'}^T\y_{a'}/n^2$.
\ENDFOR
\ENDFOR
\STATE \textbf{Output}: $\hat{H}$.
\caption{Estimating $\beta_a^T\Sigma\beta_{a'}$, General covariance [Proposition~\ref{prop:poly}]}\label{alg:Hgen}
\label{alg:gen}
\end{algorithmic}
\end{algorithm}



The following is the main proposition for Algorithm~\ref{alg:Hgen}. Note  $\frac{1}{\binom{n}{2}}\y_a^T(X_aX_a^T)_{up}\y_a$ is an unbiased estimator of $\beta_a^T\Sigma^{2}\beta_a$, and $\frac{1}{n^2}\y_a^TX_aX_{a'}^T\y_{a'}$ is an unbiased estimator of $\beta_a^T\Sigma^{2}\beta_{a'}$. For any $t\ge 1$, $\frac{1}{\binom{n}{2}}\y_a^T(X_a\frac{X^TG^{t-1}X}{\binom{s}{t}}X_a^T)_{up}\y_a$ is an unbiased estimator of $\beta_a^T\Sigma^{2+t}\beta_a$, and $\frac{1}{n^2}\y_a^TX_a\frac{X^TG^{t-1}X}{\binom{s}{t}}X_{a'}^T\y_{a'}$ is an unbiased estimator of $\beta_a^T\Sigma^{2+t}\beta_{a'}$. Proposition 3 of~\cite{kong2018estimating} provides a degree $k$ polynomial with approximation error $\min(\frac{2}{k^2},2e^{-1(k-1)\sqrt{\frac{{\sigma_{min}}}{{\sigma_{max}}}}})$ in the interval $[\sigma_{\min},\sigma_{\min}]$. Given accurate estimation of $\beta_a^T\Sigma^{2}\beta_{a'}, \beta_a^T\Sigma^{3}\beta_{a'}, \beta_a^T\Sigma^{4}\beta_{a'},\ldots$, one can linearly combine these estimates to approximate $\beta_a\Sigma\beta$ where the coefficients correspond to the coefficients of $x^2,x^3,x^4,\ldots$ in the polynomial provided by Proposition 3 of~\cite{kong2018estimating}. We plug in such a polynomial to Algorithm~\ref{alg:Hgen} and obtain the following proposition on the approximation of diagonal entry $H_{a,a} = \beta_a\Sigma\beta_{a}$ and off-diagonal entry $H_{a,a'} = \beta_a\Sigma\beta_{a'}$.
\begin{proposition}\label{prop:poly}
Let $p(x)$ be a degree $k+2$ polynomial $p(x)=\sum_{i=0}^{k} a_i x^{i+2}$ that approximates the function $f(x)=x$ for all $x\in [\sigma_{\min},\sigma_{\max}],$ where $\sigma_{\min}I_d\preceq \Sigma\preceq I_d\sigma_{\max}$. Let $P = a_0I_d + \sum_{t=1}^k \frac{a_t}{\binom{s}{t}}X^TG^{t-1}X$ be the matrix $P$ defined in Algorithm~\ref{alg:Hgen}. We have that for any $a\ne a'$,
\begin{align*}
&\E[(\frac{\y_a^T(X_aPX_a^T)_{up}\y_a}{\binom{n}{2}}-\beta_a\Sigma\beta_a)^2]\\
&\le \min(\frac{4}{k^4},4e^{-2(k-1)\sqrt{\frac{{\sigma_{min}}}{{\sigma_{max}}}}})+f(k)\max(\frac{d^{k}}{s^{k}},1)\frac{d+n}{n^2}, 
\end{align*}
and 
\begin{align*}
&\E[(\frac{\y_a^TX_aPX_{a'}^T\y_{a'}}{n^2}-\beta_a\Sigma\beta_{a'})^2]\\
&\le \min(\frac{4}{k^4},4e^{-2(k-1)\sqrt{\frac{{\sigma_{min}}}{{\sigma_{max}}}}})+f(k)\max(\frac{d^{k}}{s^{k}},1)\frac{d+n}{n^2},
\end{align*}
for $f(k) = k^{O(k)}.$
\end{proposition}

\section{Experiments}
We now briefly provide some empirical indication of the benefit of our approach. In all these experiments, we consider the known covariance setting. Note that as long as prior data about contexts is available, as it will commonly be in consumer, health and many other applications, it would be possible to estimate the covariance in advance. 

We first present results in a synthetic contextual multi-armed bandits setting. There are $K=5$ arms, and the  input context vectors are drawn from a normal distribution with 0 mean and an identity covariance matrix.  Our results are displayed in Figure~\ref{fig:plots} for context vectors of dimension 500, 2,000 and 50,000. Here our aim is to illustrate that we are able to estimate the optimal reward accurately after seeing significant fewer contexts than would be required by the standard alternative approach for contextual bandits which would try to estimate the optimal policy, and then estimate the performance of that optimal policy. More precisely, in this setting we use the linear disjoint contextual bandits algorithm~\cite{li2010contextual} to estimate the betas and covariance for each arm (with an optimally chosen regularization parameter in the settings where $n<d$). We then define the optimal policy as the best policy given those empirical estimates. We show the true reward of this learned policy.

\begin{figure*}[!h]
\centering
    \includegraphics[width=1\linewidth]{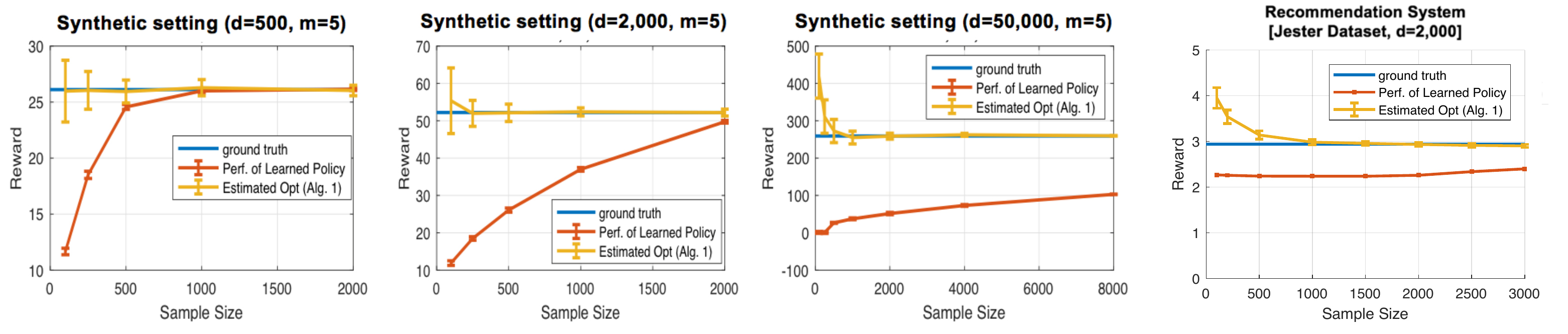}
    \vspace{-0.3in}
    \caption{The three synthetic data plots depict our algorithm for estimating the optimal reward in a synthetic domain with dimension $d=500$ (left), $d$=2,000 (center), and $d=$50,000 (right) in the setting with $m=5$ arms corresponding to independently chosen vectors $\beta_1,\ldots,\beta_{5} \in \mathbb{R}^d$ with entries chosen independently from $N(0,1).$   Our estimated value of the optimal reward is accurate when the sample size is significantly less than $d$, a regime where the best learned policy does not accurately represent the optimal policy. The right plot depicts optimal reward estimation for a  recommendation system that recommends one of 10 jokes (arms), where features are based on evaluations of 90 other jokes, represented in a $d=2000$ dimensional space.  
    In each plot the blue line corresponds to the true reward of the optimal policy. and the red lines depicts the performance of the learned policy at that sample size using disjount linUCB. 
    \label{fig:plots}}
\end{figure*}

We also present results for a real-world setting that mimics a standard recommendation platform trying to choose which products to recommend to a user, given a high-dimensional featurization for that user.  Our experiment is based on the Jester dataset~\cite{goldberg2001eigentaste}. This is a well studied dataset which includes data for >70,000 individuals providing ratings for 100 jokes. We frame this a multi-armed bandit setting by holding out the 10 most-rated jokes, and seeking to learn a policy to select which of these jokes to offer to a particular input user, based on a feature set that captures that user's preferences based on the ratings for the remaining 90 jokes. We keep a set of 48447 users who rated all the 10 most popular jokes. For each person, we create a $d=2000$ dimensional feature vector by multiplying their 90-dimensional vector of joke ratings (with missing entries replaced by that user's average rating) by a random $90 \times 2000$ matrix (with i.i.d. $N(0,1)$ entries), and then applying a sigmoid to each of the resulting values.  The reward is the user's reported rating for the joke selected by the policy. We found that the optimal expected linear policy value using this featurization was 2.98 (out of a range of 0 to 5).  For comparison, the same approach with $d=100$ has optimal policy with value 2.81, reflecting the fact that linear functions of the lower dimensional featurization cannot capture the preferences of the user as accurately as the higher dimensional featurization.  Even for $d=2000,$ the full dataset of $\approx 50,000$ people is sufficient to accurately estimate this ``ground truth'' optimal policy.  Based on this $d=2000$ representation of the user's context, we find that even with $n=500$ contexts, we can accurately estimate the optimal reward of the best threshold policy, to within about 0.1 accuracy, which improves significant for $n\ge 1000$ (Figure~\ref{fig:plots} (right)). Note that this is significantly lower than we would need to compute any optimal policy.

We also evaluated our algorithm on NCI-60 Cancer Growth Inhibition dataset, where the cell growth inhibition effect is recorded for different types of chemical compounds tested on 60 different cancer cell lines with different concentration levels. We picked $26555$ types of chemicals that are tested on the NCI-H23 (non-small cell lung cancer) cell line with concentration level: $-4, -5, -6, -7, -8$ log10(M). We obtain the $1000$-dimensional Morgan Fingerprints representation of each chemical from its SMILES representation using the Morgan algorithm implemented in RDKit. The task is to choose the most effective concentration level (among the five concentration levels) for the chemical compound, given the high-dimensional feature representation of the compound. We re-scaled the cancer inhibition effect as between $0$ and $200$, where $0$ means no growth inhibition, $100$ means completion growth inhibition, and $200$ means the cancer cells are all dead. Figure~\ref{fig:cancer} depicts the result of running our algorithm and LinUCB algorithm~\cite{li2010contextual}. The blue line depicts the true reward ($65.29$) of the optimal policy estimated from all $26555$ datapoints. The red line depicts the average reward and confidence interval over the last $100$ rounds by executing the LinUCB algorithm with $\alpha=1$ and different sample size. Notice that the LinUCB algorithm is fully adaptive and a given sample size $n$ in Figure~\ref{fig:cancer} actually corresponds to running LinUCB algorithm for $5n$ rounds.  Unlike our algorithm which achieves an accurate estimation with roughly $500$ samples per arm, LinUCB is unable to learn a good policy even with $5\times 4000 = 20000$ adaptive rounds. In this example, there is very little linear correlation between the feature of the chemical compound and the inhibition effect, and simply always choosing the highest concentration achieves near-optimal reward. However, it takes thousands of rounds for the disjoint LinUCB algorithm to start playing near optimally. 
\vspace{-0.1in}
\begin{figure}[h]\label{fig:cancer}
\centering
    \includegraphics[width=0.5\linewidth]{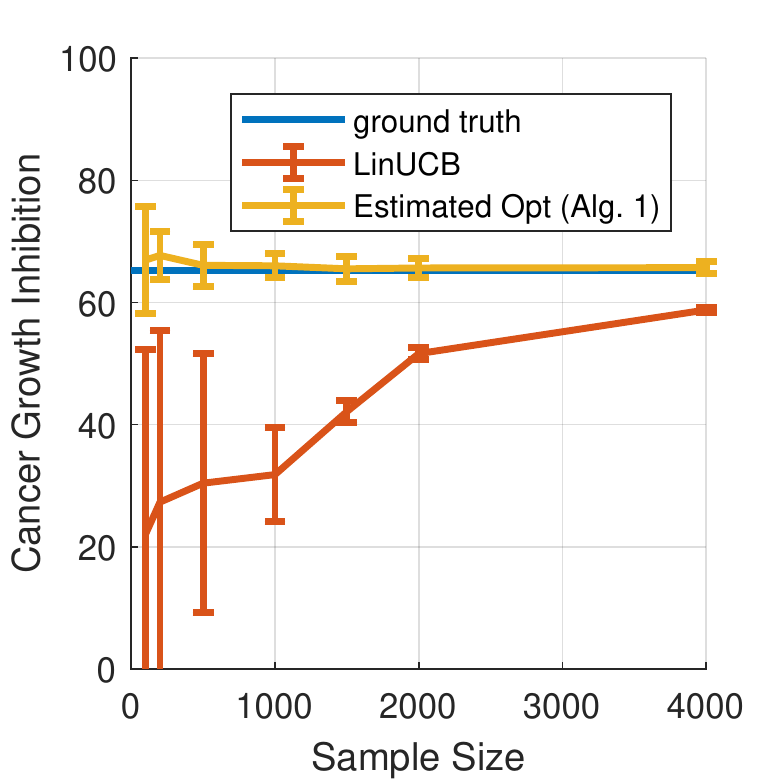}
    \vspace{-0.1in}
    \caption{Evaluation on NCI-60 growth inhibition data. 
    }
\end{figure}

\section{Conclusion}
To conclude, we present a promising approach for estimating the optimal reward in linear disjoint contextual bandits using a number of samples that is sublinear in the input contextual dimension. Without further assumptions a linear number of samples is required to output a single potentially optimal policy. There exist many interesting directions for future work, including considering more generic contextual bandit settings with an infinite set of arms.

\section*{Acknowledgments}
The contributions of Weihao Kong and Gregory Valiant were partially supported by a Google Faculty Fellowship, an Amazon Faculty Fellowship, and by NSF award 1704417 and an ONR Young Investigator award. Emma Brunskill was supported in part by a NSF CAREER award. 

\bibliography{ref}
\bibliographystyle{plain}

\section{Proof of Proposition~\ref{prop:bH2OPT}}
\begin{proof}[Proof of Proposition~\ref{prop:bH2OPT}]
The following lemma is a restatement of Theorem 1.2 of~\cite{chatterjee2005error} which bound the change of the expected maximum by the entry-wise perturbation of the covariance matrix.
\begin{lemma}[Theorem 1.2 of~\cite{chatterjee2005error}]\label{lemma:chatterjee}
Let $H\in R^{m\times m}, H'\in R^{m\times m}$ be two PSD matrices, and $\b\in R^m$ be a $m$-dimensional real vector. Let $\gamma = \max_{i,j}|H_{i,j}-H'_{i,j}|$, then 
$$
|\E_{\x\sim N(\b, H)}[\max_i x_i]-\E_{\x\sim N(\b, H')}[\max_i x_i]|\le 2\sqrt{\gamma\log m}.
$$
\end{lemma}
Lemma~\ref{lemma:chatterjee} handles the perturbation of the covariance matrix. The following simple proposition handles the perturbation of the mean, which, combined with Lemma~\ref{lemma:chatterjee}, immediately implies the statement of our proposition.

\begin{lemma}\label{lemma:mean-perturb}
Let $H\in R^{m\times m}$ be a PSD matrices, and $\b, \b' \in R^m$ be two $m$-dimensional real vectors. Then 
$$
|\E_{\x\sim N(\b, H)}[\max_i x_i]-\E_{\x\sim N(\b', H)}[\max_i x_i]|\le \max_i|b_i-b'_i|.
$$
\end{lemma}
\begin{proof}[Proof of Lemma~\ref{lemma:mean-perturb}]
Let $\x\sim N(\b,H)$ and $\x' = \x+\b'-\b$. Then the random vector $\x'$ follows from $N(\b',H)$. We have
$$
|\E[\max_i x_i]-\E[\max_i x'_i]|\le \E[|\max_i x_i-\max_i x'_i|]\le \max_i|b_i-b'_i|,
$$
which concludes the proof.
\end{proof}
Combining the two lemma, we have that 
$$
|\E_{\x\sim N(\b, H)}[\max_i x_i]-\E_{\x\sim N(\b', H')}[\max_i x_i]|\le 2\sqrt{\max_{i,j}|H_{i,j}-H'_{i,j}|\log m}+ \max_i|b_i-b'_i|,
$$
which concludes the proof.
\end{proof}
\section{Proofs of the Upper Bounds in the Known Covariance Setting}
\begin{proof}[Proof of Theorem~\ref{thm:iso}]
Applying Fact~\ref{fact:boosting} on top of Corollary~\ref{cor:Hiso}, we have that for a fixed $i,j$, with probability at least $1-\exp(-\log(K^2/\delta-1)\ge 1-\delta/(K^2+K)$, the median estimates $\hat{H}$ of Algorithm~\ref{alg:main} satisfies $|\hat{H}_{i,j}-H_{i,j}|\le3 \sqrt{\frac{9d+3n}{n^2}}\sigma_i\sigma_j$. We define $\bar{\sigma}$ such that $\sigma_i\le \bar{\sigma}$ for all $i$. Applying Fact~\ref{fact:boosting} with Corollary~\ref{cor:var-b}, we get that for a fixed $i$, with probability at least $1-\exp(-\log (K^2/\delta)-1)\ge 1-\delta/(K^2+K)$, the median estimates $\hat{\b}$ of Algorithm~\ref{alg:main} satisfies $|\hat{b}_{i}-b_{i}|\le 3\sqrt{\frac{1}{n}}\sigma_i$. Hence by a union bound, we have that with probability at least $1-\delta$, $\hat{H}$ and $\hat{\b}$ satisfy
$$
\max_{i,j} |\hat{H}_{i,j}-H_{i,j}|\le 3\sqrt{\frac{9d+3n}{n^2}}\bar{\sigma}^2,
$$
$$
\quad  \max_{i} |\hat{b}_{i}-b_{i}|\le 3\sqrt{\frac{1}{n}}\bar{\sigma}.
$$
In order to bound the discrepancy between $\hat{H}^{(PSD)}$ and $H$, notice that by the optimality of $\hat{H}^{(PSD)}$, there is $\max |\hat{H}^{(PSD)}_{i,j}-\hat{H}_{i,j}|\le \max |H_{i,j}-\hat{H}_{i,j}|$. Applying triangle inequality, we have 
\begin{align*}
&\max |\hat{H}^{(PSD)}_{i,j}-H_{i,j}|\\
\le& \max |\hat{H}^{(PSD)}_{i,j}-\hat{H}_{i,j}|+\max |H_{i,j}-\hat{H}_{i,j}|\\
\le& 6\sqrt{\frac{9d+3n}{n^2}}\bar{\sigma}^2
\end{align*} 
Thus, by Proposition~\ref{prop:bH2OPT}, with probability $1-\delta$ the final estimation error is bounded by
\begin{align*}
|OPT-\widehat{OPT}| &\le 7\sqrt{\log K}(\frac{3d+n}{n^2})^{1/4}\bar{\sigma}+3\frac{1}{\sqrt{n}}\bar{\sigma}\\
&\le 10\sqrt{\log K}(\frac{3d+n}{n^2})^{1/4}\bar{\sigma}\\
& = O(\sqrt{\log K}(\frac{d+n}{n^2})^{1/4}),
\end{align*}
where we have apply the fact that $\bar{\sigma}$ is a constant.
\end{proof}
Corollary~\ref{cor:isotropic} follows immediately from Theorem~\ref{thm:iso}.
\begin{proof}[Proof of Corollary~\ref{cor:isotropic}]
In order to achieve additive error $\eps$, we set $n = \Theta(\frac{\log K}{\eps^2} \max(\sqrt{d}, \frac{\log K}{\eps^2})) = \Theta(\frac{\sqrt{d}\log K}{\eps^2})$ where the last equality holds by Theorem~\ref{thm:iso} and the assumption on $\eps$. Algorithm~\ref{alg:main} in total requires $T = \Theta(nK(\log K+\log(1/\delta))) = \Theta(\frac{\sqrt{d}K\log K}{\eps^2}(\log K+\log(1/\delta)))$ samples.
\end{proof}

\section{Proofs of the Upper Bounds in the Unknown Covariance Setting}
\subsection{Proof of Proposition~\ref{prop:poly}, Estimating $H$ in the Unknown Covariance Setting.}
In order to prove Proposition~\ref{prop:poly}, we first prove Proposition~\ref{prop:mom-aa} and Proposition~\ref{prop:mom-aap}, where Proposition~\ref{prop:mom-aa} gives a variance bound for our estimator of $\beta_a\Sigma^{k+2}\beta_a$ for $k\ge 0$, and Proposition~\ref{prop:mom-aap} gives a variance bound for our estimator of $\beta_a\Sigma^{k+2}\beta_{a'}$ for $k\ge 0$. Then Proposition~\ref{prop:poly} holds by combining the two propositions.
\begin{proposition}\label{prop:mom-aa}
For arm $a$, We denote $\{\x_i\}$ as a set of unlabeled examples, where $|{\x_i}|=s$.
$$
\E[\frac{1}{\binom{s}{k}\binom{n}{2}}\sum_{i<j}y_{a,i}y_{a,j}\x_{a,i}^T\left( \sum_{i_1<i_2<\ldots<i_{k}}\x_{i_1}\x_{i_1}^T\x_{i_2}\x_{i_2}^T\ldots\x_{i_{k}}\x_{i_k}^T\right)\x_{a,j})] = \beta_a^T\Sigma^{k+2}\beta_a
$$
$$
\Var[\frac{1}{\binom{s}{k}\binom{n}{2}}\sum_{i<j}y_{a,i}y_{a,j}\x_{a,i}^T\left( \sum_{i_1<i_2<\ldots<i_{k}}\x_{i_1}\x_{i_1}^T\x_{i_2}\x_{i_2}^T\ldots\x_{i_{k}}\x_{i_k}^T\right)\x_{a,j})] = f(k)\max(\frac{d^{k}}{s^{k}},1)\frac{d+n}{n^2}),
$$ where $f(k)=k^{O(k)}$
\end{proposition}
\begin{proof}
Notice that for $i< j$,  $y_{a,i}y_{a,j}\x_{a,i}^T\left( \sum_{i_1<i_2<\ldots<i_{k}}\x_{i_1}\x_{i_1}^T\x_{i_2}\x_{i_2}^T\ldots\x_{i_{k}}\x_{i_k}^T\right)\x_{a,j}$ is an unbiased estimator for $\beta_a^T\Sigma^{k+2}\beta_a$. Since the average of unbiased estimators is still an unbiased estimator, the proposition statement about the expectation holds. We write the variance of the estimator as follows,
\begin{align}
\frac{1}{\binom{n}{2}^2}\sum_{i<j,i'<j'}(\E[y_{a,i}y_{a,j}\x_{a,i}^T\left( \frac{1}{\binom{s}{k}}\sum_{i_1<i_2<\ldots<i_{k}}\x_{i_1}\x_{i_1}^T\x_{i_2}\x_{i_2}^T\ldots\x_{i_{k}}\x_{i_k}^T\right)\x_{a,j}y_{a,i'}y_{a,j'}\x_{a,i'}^T\\
\left(\frac{1}{\binom{s}{k}} \sum_{i_1<i_2<\ldots<i_{k}}\x_{i_1}\x_{i_1}^T\x_{i_2}\x_{i_2}^T\ldots\x_{i_{k}}\x_{i_k}^T\right)\x_{a,j'}]-(\beta_a^T\Sigma^{k+2}\beta_a)^2 )\\
=\frac{1}{\binom{n}{2}^2}\sum_{i<j,i'<j'}(\E[y_{a,i}y_{a,j}\x_{a,i}^T\left(\frac{1}{\binom{s}{k}} \sum_{i_1<i_2<\ldots<i_{k}}\x_{i_1}\x_{i_1}^T\x_{i_2}\x_{i_2}^T\ldots\x_{i_{k}}\x_{i_k}^T\right) \x_{a,j}y_{a,i'}y_{a,j'}\x_{a,i'
}^T\\
\left(\frac{1}{\binom{s}{k}} \sum_{i_1<i_2<\ldots<i_{k}}\x_{i_1}\x_{i_1}^T\x_{i_2}\x_{i_2}^T\ldots\x_{i_{k}}\x_{i_k}^T\right)\x_{a,j'}-y_{a,i}y_{a,j}\x_{a,i}^T\Sigma^k\x_{a,j}y_{a,i'}y_{a,j'}\x_{a,i'}^T\Sigma^k\x_{a,j'}]\\
+\E[y_{a,i}y_{a,j}\x_{a,i}^T\Sigma^k\x_{a,j}y_{a,i'}y_{a,j'}\x_{a,i'}^T\Sigma^k\x_{a,j'}] - (\beta_a^T\Sigma^{k+2}\beta_a)^2).
\end{align}

For each term in the summation, we classify it into one of the 3 different cases according to $i,j,i',j'$:
\begin{enumerate}
\item If $i\ne i'$ and $j\ne j'$, the term is $0$.
\item If $i=i'$ and $j\ne j'$, the term can be written as $\E[y_{a,i}^2\x_{a,i}^T\left( \sum_{i_1<i_2<\ldots<i_{k}}\x_{i_1}\x_{i_1}^T\x_{i_2}\x_{i_2}^T\ldots\x_{i_{k}}\x_{i_k}^T\right)\Sigma\beta\\
\x_{a,i
}^T\left( \sum_{i_1<i_2<\ldots<i_{k}}\x_{i_1}\x_{i_1}^T\x_{i_2}\x_{i_2}^T\ldots\x_{i_{k}}\x_{i_k}^T\right)\Sigma\beta-y_{a,i}^2\x_{a,i}^T\Sigma^{k+1}\beta \x_{a,i}^T\Sigma^{k+1}\beta] + \E[y_{a,i}^2\x_{a,i}^T\Sigma^{k+1}\beta \x_{a,i}^T\Sigma^{k+1}\beta] -(\beta_a^T\Sigma^{k+2}\beta_a)^2)$. By Lemma 2 of~\cite{kong2018estimating}, the first expectation is bounded by $f(k)\max(\frac{d^{k-1}}{s^k},\frac{1}{s})\E[y_{a,i}^2\x_{a,i}^T\x_{a,i}] = O(f(k)\max(\frac{d^{k}}{s^k},\frac{d}{s}))$, and the second difference is bounded by constant by the four moment condition of Gaussian. 
\item If $i\ne i'$ and $j=j'$, this case is symmetric to case 2.
\item If $i=i'$ and $j=j'$, the term can be written as $\E[y_{a,i}^2y_{a,j}^2\x_{a,i}^T\left( \sum_{i_1<i_2<\ldots<i_{k}}\x_{i_1}\x_{i_1}^T\x_{i_2}\x_{i_2}^T\ldots\x_{i_{k}}\x_{i_k}^T\right)\x_{a,j}\\
\x_{a,i
}^T\left( \sum_{i_1<i_2<\ldots<i_{k}}\x_{i_1}\x_{i_1}^T\x_{i_2}\x_{i_2}^T\ldots\x_{i_{k}}\x_{i_k}^T\right)\x_{a,j}-(y_{a,i}y_{a,j}\x_{a,i}^T\Sigma^{k}\x_{a,j})^2 ] + \E[(y_{a,i}y_{a,j}\x_{a,i}^T\Sigma^{k}\x_{a,j})^2] -(\beta_a^T\Sigma^{k+2}\beta_a)^2)$. By Lemma 2 of~\cite{kong2018estimating}, the first expectation is bounded by $$
f(k)\max(\frac{d^{k-1}}{s^k},\frac{1}{s})\E[y_{a,i}^2\x_{a,i}^T\x_{a,i}]\E[y_{a,j}^2\x_{a,j}^T\x_{a,j}] = O(f(k)d\max(\frac{d^{k}}{s^k},\frac{d}{s}))
$$, and the second difference is bounded by $O(d)$ by the four moment condition of Gaussian.
\end{enumerate}
The final step is to sum the contributions of these $3$ cases. Case $2$ has $O(n^3)$ different quadruples $(i,j,i',j')$. Case $4$ has $n^2$ different quadruples $(i,j,i',j')$. Combining the resulting bounds yields the following bound on the variance:
$
\frac{1}{\binom{n}{2}^2}(n^3\max(\frac{d^k}{s^k},\frac{d}{s},1)+n^2d\max(\frac{d^k}{s^k},\frac{d}{s},1)) = f(k)\max(\frac{d^{k}}{s^{k}},1)\frac{d+n}{n^2}.
$
\end{proof}
\begin{proposition}\label{prop:mom-aap}
For arm $a\ne a'$, let $\mu_a = \frac{\sum_{j}y_{a,j}\x_{a,j}}{n}, \mu_{a'} = \frac{\sum_{j}y_{a',j}\x_{a',j}}{n}$. We denote $\{\x_i\}$ as a set of unlabeled examples, where $|{\x_i}|=s$.
$$
\E[\frac{1}{\binom{s}{k}}\mu_a^T\left(\sum_{i_1<i_2<\ldots<i_{k}}\x_{i_1}\x_{i_1}^T\x_{i_2}\x_{i_2}^T\ldots\x_{i_{k}}\x_{i_k}^T\right)\mu_{a'}] = \beta_a^T\Sigma^{k+2}\beta_{a'}
$$
$$
\Var[\frac{1}{\binom{s}{k}}\mu_a^T\left(\sum_{i_1<i_2<\ldots<i_{k}}\x_{i_1}\x_{i_1}^T\x_{i_2}\x_{i_2}^T\ldots\x_{i_{k}}\x_{i_k}^T\right)\mu_{a'}] = f(k)\max(\frac{d^{k}}{s^{k}},1)\frac{d+n}{n^2},
$$where $f(k)=k^{O(k)}$.
\end{proposition}

\begin{proof}
Notice that $\E[\mu_a] = \Sigma\beta_a$.
It's easy to see that 
$$
\E[\mu_a\x_{i_1}\x_{i_1}^T\ldots\x_{i_k}\x_{i_k}^T\mu_{a'}] = \beta_a\Sigma^{k+2}\beta_{a'}.
$$
For the variance bound, we can express the variance as the summation of the following two terms,
\begin{align*}
\E_{\mu_a,\mu_{a'}}[\E_{\x}[(\frac{1}{\binom{s}{k}}\mu_a^T\left(\sum_{i_1<i_2<\ldots<i_{k}}\x_{i_1}\x_{i_1}^T\x_{i_2}\x_{i_2}^T\ldots\x_{i_{k}}\x_{i_k}^T\right)\mu_{a'})^2]-(\mu_a\Sigma^{k+2}\mu_{a'})^2]\\
+\E_{\mu_a,\mu_{a'}}[(\mu_a\Sigma^{k}\mu_{a'})^2-(\beta_a\Sigma^{k+2}\beta_{a'})^2]
\end{align*}
The first term, by Lemma 2 of~\cite{kong2018estimating} and Fact~\ref{fact:mu-a-bound}, is bounded by 
$$
f(k)\min(\frac{d^{k-1}}{s^k},\frac{1}{s})\E_{\mu_a,\mu_{a'}}[\|\mu_a\|^2\|\mu_{a'}\|^2]\le f(k)\max(\frac{d^{k-1}}{s^{k}},\frac{1}{s})\max(\frac{d^2}{n^2},1).
$$
The second term, by Proposition~\ref{prop:muaAmub}, is bounded by $O(\frac{d+n}{n^2})$, and summing up the two bounds yields the desired variance bound. 
\end{proof}

Before proving Proposition~\ref{prop:poly}, we first briefly show that the quantity computed in Algorithm~\ref{alg:Hgen} is equivalent to the estimators appear in Proposition~\ref{prop:mom-aa} and Proposition~\ref{prop:mom-aap}. 

\begin{fact}
For any $t\ge 1$, $\frac{1}{\binom{n}{2}}\y_a^T(X_a\frac{X^TG^{t-1}X}{\binom{s}{t}}X_a^T)_{up}\y_a = \beta_a^T\Sigma^{2+t}\beta_a$ and $\frac{1}{n^2}\y_a^TX_a\frac{X^TG^{t-1}X}{\binom{s}{t}}X_{a'}^T\y_{a'} =\beta_a^T\Sigma^{2+t}\beta_{a'}$
\end{fact}
\begin{proof}
Denote $A = (X_a\frac{X^TG^{t-1}X}{\binom{s}{t}}X_a^T)_{up}$. $A_{i,j}$ can be expanded as $$\frac{1}{\binom{s}{t}}\sum_{i_1,i_2,\ldots,i_t}\x^T_{a,i}\x_{i_1}\x^T_{i_1}\x_{i_2}\x^T_{i_2}\x_{i_3}\ldots\x_{i_t}\x_{a,j}$$.
Since $G$ is an upper triangular matrix, the summation is equivalent to $$
\frac{1}{\binom{s}{t}}\sum_{i_1<i_2<\ldots<i_t}\x^T_{a,i}\x_{i_1}\x^T_{i_1}\x_{i_2}\x^T_{i_2}\x_{i_3}\ldots\x_{i_t}\x_{a,j}.
$$ We can further expand $\frac{1}{\binom{n}{2}\binom{s}{t}}\y_a^TA\y_a$ as $\frac{1}{\binom{n}{2}\binom{s}{t}}\sum_{i,j}y_{a,i}A_{i,j}y_{a,j} = \frac{1}{\binom{n}{2}\binom{s}{t}}\sum_{i<j}y_{a,i}A_{i,j}y_{a,j}$ since $A$ is an upper triangular matrix. Plugging in the expansion of $A_{i,j}$, we get the expansion
$$
\frac{1}{\binom{s}{k}\binom{n}{2}}\sum_{i<j}y_{a,i}y_{a,j}\x_{a,i}^T\left( \sum_{i_1<i_2<\ldots<i_{t}}\x_{i_1}\x_{i_1}^T\x_{i_2}\x_{i_2}^T\ldots\x_{i_{t}}\x_{i_t}^T\right)\x_{a,j})
$$, which by Proposition~\ref{prop:mom-aa} is an unbiased estimator of $\beta_a\Sigma^{t+2}\beta_a$. The case for $\beta_a\Sigma^{t+2}\beta_{a'}$ can be proved analogously.
\end{proof}
We restate Proposition~\ref{prop:poly} as follows:

\smallskip
\vspace{.2cm}\noindent \textbf{Proposition~\ref{prop:poly}.} \emph{ Let $p(x)$ be a degree $k+2$ polynomial $p(x)=\sum_{i=0}^{k} a_i x^{i+2}$ that approximates the function $f(x)=x$ for all $x\in [\sigma_{\min},\sigma_{\max}],$ where $\sigma_{\min}$ and $\sigma_{\max}$ are the minimum and maximum singular values of $\Sigma$. Let $P = a_0I_d + \sum_{t=1}^k \frac{a_t}{\binom{s}{t}}X^TG^{t-1}X$ be the matrix $P$ defined in Algorithm~\ref{alg:Hgen}. We have that for any $a\ne a'$,
$$
\E[(\frac{\y_a^T(X_aPX_a^T)_{up}\y_a}{\binom{n}{2}}-\beta_a\Sigma\beta_a)^2]\le \min(\frac{4}{k^4},4e^{-2(k-1)\sqrt{\frac{{\sigma_{min}}}{{\sigma_{max}}}}})+f(k)\max(\frac{d^{k}}{s^{k}},1)\frac{d+n}{n^2}.
$$
$$
\E[(\frac{\y_a^TX_aPX_{a'}^T\y_{a'}}{\binom{n}{2}}-\beta_a\Sigma\beta_{a'})^2]\le \min(\frac{4}{k^4},4e^{-2(k-1)\sqrt{\frac{{\sigma_{min}}}{{\sigma_{max}}}}})+f(k)\max(\frac{d^{k}}{s^{k}},1)\frac{d+n}{n^2}.
$$
for $f(k) = k^{O(k)}$.
}

\begin{proof}[Proof of Proposition~\ref{prop:poly}]
Notice that $P = a_0I_d + \sum_{t=1}^k \frac{a_t}{\binom{s}{t}}X^TG^{t-1}X$. By definition, we have
\begin{align*}
&\E[(\frac{\y_a^T(X_aPX_a^T)_{up}\y_a}{\binom{n}{2}}-\beta_a\Sigma\beta_a)^2] \\
&= \E[\biggl(a_0\frac{\y_a^T(X_aX_a^T)_{up}\y_a}{\binom{n}{2}}-\beta_a^T\Sigma^2\beta_a+\sum_{t=1}^k\left(a_t\frac{a_t\y_a^T(X_aX^TG^{t-1}XX_a^T)_{up}\y_a}{\binom{n}{2}}-a_t\beta_a^T\Sigma^{2+t}\beta_a\right)\\
&+\sum_{t=0}^k\beta_a^T\Sigma^{2+t}\beta_a-\beta_a\Sigma\beta_a\biggr)^2]\\
&= \E[\left(a_0\frac{\y_a^T(X_aX_a^T)_{up}\y_a}{\binom{n}{2}}-\beta_a^T\Sigma^2\beta_a+\sum_{t=1}^k\left(a_t\frac{a_t\y_a^T(X_aX^TG^{t-1}XX_a^T)_{up}\y_a}{\binom{n}{2}}-a_t\beta_a^T\Sigma^{2+t}\beta_a\right)\right)^2]\\
&+\biggl(\sum_{t=0}^k\beta_a^T\Sigma^{2+t}\beta_a-\beta_a\Sigma\beta_a\biggr)^2,
\end{align*}
where in the last inequality we use the unbiasedness of these estimators. By Proposition~\ref{prop:mom-aa}, Proposition~\ref{prop:mom-aap} and due to the fact that for any random variable $X_1,X_2,\ldots,X_k$, $\E[(X_1+X_2+\ldots+X_k)^2] = \sum_{i,j}\E[X_iX_j] \le \sum_{i,j}\sqrt{\E[X_i^2]\E[X_j^2]} = (\sum_{i}\sqrt{\E[X_i^2]})^2$. The above equation is bounded by 
\begin{align*}
\le k^2f'(k)\max(\frac{d^{k}}{s^{k}},\frac{d}{s})\frac{d+n}{n^2}+\biggl(\sum_{t=0}^k\beta_a^T\Sigma^{2+t}\beta_a-\beta_a\Sigma\beta_a\biggr)^2\\
\le f(k)\max(\frac{d^{k}}{s^{k}},1)\frac{d+n}{n^2}+ \min(\frac{4}{k^4},4e^{-2(k-1)\sqrt{\frac{{\sigma_{min}}}{{\sigma_{max}}}}})
\end{align*} where we have applied Proposition 3 of~\cite{kong2018estimating} in the last inequality. The case for $\beta_a\Sigma\beta_{a'}$ can be handled analogously, and this concludes the proof.
\end{proof}
The following are the auxiliary propositions that facilitate the proof in this sections.
\begin{fact}\label{fact:mu-a-bound}
For each arm $a$, let $\mu_a = \frac{\sum_{j}y_{a,j}\x_{a,j}}{n}$. Then $\E[\|\mu_a\|^2] = O(\frac{d+n}{n}) $
\end{fact}
\begin{proof}
$\E[\|\mu_a\|^2] = \frac{1}{n^2}(\sum_{i\ne j}y_{a,i}y_{a,j}\x_{a,i}^T\x_{a,j} + \sum_i y_{a,i}^2\x_{a,i}^T\x_{a,i}) = O(\|\beta_a\|^2) + \frac{1}{n^2}\sum_i y_{a,i}^2\x_{a,i}^T\x_{a,i}\le O((1+\frac{nd}{n^2})(\|\beta_a\|^2+\sigma^2)) = O(\frac{n+d}{n})$.
\end{proof}
The following proposition is a slightly more stronger version of Proposition~\ref{prop:Hkkp-iso}.
We omit the proof since it is almost identical to the proof of Proposition~\ref{prop:Hkkp-iso}.
\begin{proposition}\label{prop:muaAmub}
For arm $a\ne a'$, let $\mu_a = \frac{\sum_{j}y_{a,j}\x_{a,j}}{n}, \mu_{a'} = \frac{\sum_{j}y_{a',j}\x_{a',j}}{n}$. Let $A$ be a real $d\times d$ matrix such that$\|A\|=O(1)$. We have $\Var[\mu_a A\mu_b] =O(\frac{d+n}{n^2})$
\end{proposition}

\subsection{Proof of Theorem~\ref{thm:unknownc} and Corollary~\ref{cor:unknown}, Main Result in the Unknown Covariance Setting}

We are ready to prove the main theorem of the unknown covariance, and we restate Theorem~\ref{thm:unknownc} as follows,

\smallskip
\vspace{.2cm}\noindent \textbf{Theorem~\ref{thm:unknownc}.} \emph{In the unknown covariance setting, for any positive integer $k$, with probability $1-\delta$, Algorithm~\ref{alg:main} estimates the optimal reward OPT with additive error:
$$
|OPT-\widehat{OPT}| \le O\left(\sqrt{\log K} {\left(\min(\frac{1}{k^2},e^{-(k-1)\sqrt{\frac{{\sigma_{min}}}{{\sigma_{max}}}}})+f(k)\max(\frac{d^{k/2}}{s^{k/2}},1)\sqrt{\frac{d+n}{n^2}}\right)}^{1/2}\right),
$$
where $f(k) = k^{O(k)}$.}

\begin{proof}[Proof of Theorem~\ref{thm:unknownc}]
Applying Fact~\ref{fact:boosting} on top of Corollary~\ref{cor:Hgen}, we have that for a fixed $i,j$, with probability at least $1-\delta/(K^2+K)$, the median estimates $\hat{H}$ of Algorithm~\ref{alg:main} satisfies $$
|\hat{H}_{i,j}-H_{i,j}|\le \min(\frac{2}{k^2},2e^{-(k-1)\sqrt{\frac{{\sigma_{min}}}{{\sigma_{max}}}}})+f(k)\max(\frac{d^{k/2}}{s^{k/2}},1)\sqrt{\frac{d+n}{n^2}}.
$$
Applying Fact~\ref{fact:boosting} with Corollary~\ref{cor:var-b}, we get that for a fixed $i$, with probability at least $1-\delta/(K^2+K)$, the median estimates $\hat{\b}$ of Algorithm~\ref{alg:main} satisfies $|\hat{b}_{i}-b_{i}|= O(\sqrt{\frac{1}{n}})$. Hence by a union bound, we have that with probability at least $1-\delta$, $\hat{H}$ and $\hat{\b}$ satisfy
$$
|\hat{H}_{i,j}-H_{i,j}|\le \min(\frac{2}{k^2},2e^{-(k-1)\sqrt{\frac{{\sigma_{min}}}{{\sigma_{max}}}}})+f(k)\max(\frac{d^{k/2}}{s^{k/2}},1)\sqrt{\frac{d+n}{n^2}}.
$$
$$
\quad  \max_{i} |\hat{b}_{i}-b_{i}|\le O(\sqrt{\frac{1}{n}}).
$$
In order to bound the discrepancy between $\hat{H}^{(PSD)}$ and $H$, notice that by the optimality of $\hat{H}^{(PSD)}$, there is $\max |\hat{H}^{(PSD)}_{i,j}-\hat{H}_{i,j}|\le \max |H_{i,j}-\hat{H}_{i,j}|$. Applying triangle inequality, we have 
\begin{align*}
&\max |\hat{H}^{(PSD)}_{i,j}-H_{i,j}|\\
\le& \max |\hat{H}^{(PSD)}_{i,j}-\hat{H}_{i,j}|+\max |H_{i,j}-\hat{H}_{i,j}|\\
\le& 2|H_{i,j}-\hat{H}_{i,j}|
\end{align*} 
Thus, by Proposition~\ref{prop:bH2OPT}, with probability $1-\delta$ the final estimation error is bounded by
\begin{align*}
|OPT-\widehat{OPT}| &\le O\left(\sqrt{\log K}\left(\min(\frac{1}{k^2},e^{-(k-1)\sqrt{\frac{{\sigma_{min}}}{{\sigma_{max}}}}})+f(k)\max(\frac{d^{k/2}}{s^{k/2}},1)\sqrt{\frac{d+n}{n^2}}\right)^{1/2}\right).
\end{align*}
\end{proof}

\begin{proof}[Proof of Corollary~\ref{cor:unknown}]
Let $C$ be the constant in Theorem~\ref{thm:unknownc}. We can find constants $C_1, C_2$ such that setting $k = \min(C_1\sqrt{\log K}/\eps, \sqrt{\frac{\sigma_{\max}}{\sigma_{\min}}}(\log(\log K/\eps^2)+C_2))$ implies that 
\begin{align*}
C^2\log K\min(\frac{1}{k^2},e^{-(k-1)\sqrt{\frac{{\sigma_{min}}}{{\sigma_{max}}}}})\le \frac{\eps^2}{2}.
\end{align*}
Then we set 
$$
n=\Theta\Big(\max(\frac{(\log K)^{2/(k+2)}k^{O(1)}d^{1-1/(k+2)}}{\eps^{4/(k+2)}K^{1-2/(k+2)}}, \frac{k^{O(k)}\log K\sqrt{d}}{\eps^2})\Big), 
$$
and it can be verified that
\begin{align*}
C\log K f(k)\max(\frac{d^{k/2}}{(Kn/2)^{k/2}},1)\sqrt{\frac{d+n}{n^2}} \le \frac{\eps^2}{2},
\end{align*}
where we have applied the assumption that $\eps\ge \frac{\sqrt{\log K}}{d^{1/4}}$. Given our assumption on $k$, it is straightforward to verify that $(\log K)^{2/(k+2)}/\eps^{4/(k+2)}=O(1)$. Hence the condition on $n$ can be simplified to
$$
n=\Theta(\max(\frac{k^{O(1)}d^{1-1/(k+2)}}{K^{1-2/(k+2)}}, \frac{k^{O(k)}\log K\sqrt{d}}{\eps^2})). 
$$
Given these $n$ and $k$, it is not hard to verify that
\begin{align*}
C\sqrt{\log K} &(\min(\frac{1}{k^2},e^{-(k-1)\sqrt{\frac{{\sigma_{min}}}{{\sigma_{max}}}}})\\
&+f(k)\max(\frac{d^{k/2}}{(Kn/2)^{k/2}},1)\sqrt{\frac{d+n}{n^2}})^{1/2}\le \eps,
\end{align*}
and this concludes the proof.
\end{proof}
\subsection{Proof of Corollary~\ref{cor:gen-unlabeled}, Estimating OPT with a Large Set of Unlabeled Examples}
\begin{proof}[Proof of Corollary~\ref{cor:gen-unlabeled}]
Denote multiset $S = \{\x_1,\ldots, \x_{s}\}$ as the set of unlabeled examples where $s = \Theta((d+\log 1/\delta)\log^2 K/\eps^4)$, and $\hat{\Sigma} = \frac{1}{s}\sum_{i=1}^s\x_i\x_i^T$ as the covariance estimator. By standard matrix concentration results (e.g. Corollary 5.50 in~\cite{vershynin2010introduction}), we have that with probability $1-\delta/2$, $(1-\eps^2/\log K) I \preceq \hat{\Sigma}^{-1/2}\Sigma\hat{\Sigma}^{-1/2}\preceq (1+\eps^2/\log K)I$.

Then, we execute the known covariance version of Algorithm~\ref{alg:main} but scale each context $\x_{i,j}$ as $\hat{\Sigma}^{-1/2}\x_{i,j}$. Notice that the scaled contexts has variance $\tilde{\Sigma} := \hat{\Sigma}^{-1/2}\Sigma\hat{\Sigma}^{-1/2}$, and we define $\tilde{\beta_i} := \hat{\Sigma}^{1/2}\beta_i$ as the scaled coefficient vectors. As in the proof of Corollary~\ref{cor:isotropic}, we set $n = \Theta(\frac{\sqrt{d}\log K}{\eps^2})$ which implies with probability $1-\delta/2$, the error due to the variance is $\max_{i,j}|\hat{H}_{i,j} - \tilde{\beta_i}^T \tilde{\Sigma}^2 \tilde{\beta_j}| = O(\eps^2/\log K) $.
The bias term is bounded as
\begin{align*}
|H_{i,j}-\tilde{\beta_i}^T\tilde{\Sigma}^2\tilde{\beta_j}| &=~|\tilde{\beta_i}^T\tilde{\Sigma}\tilde{\beta_j}-\tilde{\beta_i}^T\tilde{\Sigma}^2\tilde{\beta_j}|\\
&=~O(\|\tilde{\beta_i}\|\|\tilde{\beta_j}\|\eps^2/\log K),
\end{align*}
where the last equality holds since $\max_{x\in [1-\eps^2/\log K, 1+\eps^2/\log K]}[|x^2-x|] = O(\eps^2/\log K)$. Since we assume that $\beta_i^T\Sigma\beta_i$ is bounded by a constant, $\|\tilde{\beta_i}\|^2 = \beta_i^T\hat{\Sigma}\beta_i$ is also bounded by a constant. Hence we have 
$$
\max_{i,j} |\hat{H}_{i,j}- H_{i,j}| = O(\eps^2/\log K).
$$
The remaining proof follows from the same argument in the proof of Theorem~\ref{thm:iso}.
\end{proof}
\section{Extension to the Mixture of Gaussians Setting}
\textbf{Problem setting:} In this section, we extend our result to the mixture of Gaussians setting, where we assume each context $\x$ is drawn from a known mixture of Gaussians distribution $\sum_{i=1}^M \alpha_i N(\mu_i,\Sigma_i)$, meaning that the means $\mu_i$', covariances $\Sigma_i$'s and mixing weights $\alpha_i$'s are all known to the algorithm. WLOG, we may assume that the mean of the mixture of Gaussians is $0$ and the covariance is identity, namely $\sum_{i=1}^M\alpha_i\mu_i = 0$ and $\sum_{i=1}^M\alpha_i(\mu_i\mu_i^\top +\Sigma_i) = I_d$, since we can always re-project the data to make the condition holds. As usual, we still assume that all $\|\beta_i\|$ and the variance of the noise $\sigma_i$ are bounded.

The following simple proposition shows that the optimal expected reward in the mixture of Gaussian model is simply the linear combination of the optimal expected reward for each component.
\begin{proposition}
In the setting where each context is drawn from a known mixture of Gaussians distribution $\sum_{m=1}^M \alpha_m N(\mu_m, \Sigma_m)$, the optimal reward has the following form:
$$
\sum_{m=1}^M \alpha_m\E_{\x\sim N(\b^{(m)}, H^{(m)})}[\max_{k\in [K]} x_k],
$$
where $\b^{(m)} = (\beta_1^\top \mu_m+b_1, \beta_2^\top \mu_m+b_2, \ldots, \beta_K^\top \mu_m+b_K)\in R^K$, and $H^{(m)}_{k,k'} = \beta_{k}\Sigma_m\beta_{k'}$.
\end{proposition}
\begin{proof}
We know from the single Gaussian case that the expected optimal reward for the contexts from the $m$th Gaussian component is 
$$\E_{\x\sim N(\b^{(m)}, H^{(m)})}[\max_{k\in [K]} x_k],
$$
where $\b^{(m)} = (\beta_1^\top \mu_m+b_1, \beta_2^\top \mu_m+b_2, \ldots, \beta_K^\top \mu_m+b_K)\in R^K$, and $H^{(m)}_{k,k'} = \beta_{k}\Sigma_m\beta_{k'}$. The overall optimal expected reward is the weight average of all these rewards with weights $\alpha_k$'s.
\end{proof}
In the following two propositions, we give the estimators for the parameters corresponding to each Gaussian compoennt, $\b^{(m)}$, $H^{(m)}$ and prove the corresponding variance bounds. Our estimators can be applied to the mixture of Gaussian setting since it only requires the fourth moment of the distribution of $\x$ to be bounded. Before stating our two propositions, we state the following simple fourth moment property of mixture of Gaussian distribution without proofs.
\begin{fact}
Suppose $\x\sim \sum_{i=1}^M \alpha_i N(\mu_i,\Sigma_i)$ and $\E[\x]=0$, $\E[\x\x^\top] = I_d$, it holds for all unit $d$-dimensional vectors $\u,\v$ that
$$
(\u^\top \x)^2(\v^\top \x)^2 = O(1)
$$
\end{fact}
\begin{proposition}
For each arm $k\in [K]$, and Gaussian component $m\in [M]$, 
\begin{align*}
\hat{b}^{(m)}_{k} &= \frac{1}{n}\sum_{i=1}^n y_{k,i}\x_{k,i}^\top\mu_m .
\end{align*}
Then, for all $k \in [K], m\in [M]$
\begin{align}
\E[\hat{\b}^{(m)}_{k}] &= \beta_{k}^\top \mu_m\\
\Var[\hat{\b}^{(m)}_{k}] &= O(\|\mu_m\|^2/n) 
\end{align}
\end{proposition}
\begin{proof}
The proof of the expectation part is trivial. We show the variance bound as follows:
\begin{align*}
\Var[\hat{b}_k^{(m)}] =& \E[(\frac{1}{n}\sum_{i=1}^n y_{k,i}\x_{k,i}^\top \mu_m)^2] - \E[(\frac{1}{n}\sum_{i=1}^n y_{k,i}\x_{k,i}^\top \mu_m)]^2\\
= &\frac{1}{n^2} \sum_{i=1}^n \Big((y_{k,i}\x_{k,i}^\top \mu_m)^2-(\beta_k \mu_m)^2\Big)\\
\le& O(\|\mu_m\|^2/n)
\end{align*}
\end{proof}
\begin{proposition}
For each arm $k\in [K]$, and Gaussian component $m\in [M]$, define
\begin{align*}
\hat{H}^{(m)}_{k,k} &= \frac{1}{\binom{n}{2}}\sum_{i<j} y_{k,i}y_{k,j}\x_{k,i}^\top \Sigma_m \x_{k,j},
\end{align*}
and for each pair of arms $k\ne k'$, define
\begin{align*}
\hat{H}^{(m)}_{k,k'} &= \frac{1}{n^2}\sum_{i=1}^n\sum_{j=1}^n y_{k,i}y_{k',j}\x_{k,i}^\top \Sigma_m \x_{k',j}.
\end{align*}
Then, for all $k, k'\in [K], m\in [M]$
\begin{align*}
\E[\hat{H}^{(m)}_{k, k'}] &= \beta_{k}\Sigma_m\beta_{k'}\\
\Var[\hat{H}^{(m)}_{k, k'}] &= O(\frac{\|\Sigma_m\|^2}{n}+\frac{\tr[\Sigma_m^2]}{n^2})
\end{align*}
\end{proposition}
\begin{proof}
The expectation part of the statement is trivial. We prove the variance bound as follows. $\Var[\hat{H}^{(m)}_{k, k}]$ can be expressed as \begin{align*}
&\frac{1}{n^4}\sum_{i\ne j,i'\ne j'}(\E[y_{k,i}y_{k,j}y_{k,i'}y_{k,j'}\x_{k,i}^T\Sigma_m\x_{k,j}\x_{k,i'}^T\Sigma_m\x_{k,j'}]\\
&-\E[y_{k,i}y_{k,j}\x_{k,i}^T\Sigma_m\x_{k,j}]\E[y_{k,i'}y_{k,j'}\x_{k,i'}^T\Sigma_m\x_{k,j'}]).\end{align*}
For each term in the summation, we classify it into one of the $3$ different cases according to $i,j,i',j'$:
\begin{enumerate}
\item If $i\ne i'$ and $j\ne j'$, the term is $0$.
\item If $i=i'$ and $j\ne j'$, the term can then be expressed as:
$\E[y_{k,i}^2y_{k,j}y_{k,j'}\x_{k,i}^T\Sigma_m \x_{k,j} \x_{k,i}^T\Sigma_m \x_{k,j'}] - (\beta_k^T\Sigma_m\beta_{k})^2
= \E[y_{k,i}^2(\beta_{k}^T \Sigma_m\x_{k,i})^2]- (\beta_{k}^T\Sigma_m\beta_{k})^2 \le \beta_k^\top\Sigma_m^2\beta_k \le \|\Sigma_m\|^2$, where the last inequality follows from the 4-th moment condition of mixture of Gaussian distribution.
\item If $i\ne i'$ and $j=j'$, this case is symmetric to the last case.
\item If $i=i'$ and $j=j'$, the term can then be expressed as:
$\E[y_{k,i}^2y_{k,j}^2(\x_{k,i}^T \Sigma_m\x_{k,j})^2] - (\beta_k^T\Sigma_m\beta_{k})^2$. First taking the expectation over $\x_{a',j},y_{a',j}$, we get the following upper bound 
$$
O(\E[y_{k,i}^2(\x_{k,i}^T\Sigma_m^2\x_{k,i})].)
$$. Notice that $\x_{k,i}^T\Sigma_m^2\x_{k,i}=\sum_{l=1}^dd_j^2(\v_j^\top\x_{a,i})^2$, where $d_j, \v_j$ are the eigenvalues and eigenvectors of the matrix $\Sigma_m$. Taking the expectation over the $i$th sample and applying the fourth moment condition of $\x$, we get the following bound:
$
O(\tr[\Sigma_m^2]).
$
\end{enumerate}
The final step is to sum the contributions of these $3$ cases. Case $2$ and $3$ have $O(n^3)$ different quadruples $(i,j,i',j')$. Case $4$ has $O(n^2)$ different quadruples $(i,j,i',j')$. Combining the resulting bounds yields a $O(\|\Sigma_m\|_2^2/n+\tr[\Sigma_m^2]/n^2)$ upper bound.

The $k\ne k'$ case can be proved analogously.
\end{proof}

We restate the main theorem of the mixture of Gaussians setting as follows, 

\vspace{.2cm}\noindent \textbf{Theorem~\ref{thm:main-mog}.} \emph{
Suppose each context $\x$ is drawn independently from a mixture of Gaussian distribution $\sum_{i=1}^M \alpha_i N(\mu_i,\Sigma_i)$, and the parameters $\mu_i, \Sigma_i, \alpha_i$ are all known to the algorithm. In addition, let us assume that $\|\mu_i\|$, $\|\Sigma_i\|$ are all bounded by a constant. Then, for any $\eps\ge \frac{\sqrt{\log K}}{d^{1/4}}$, with probability $1-\delta$, there is an algorithm that estimate the optimal reward $OPT$ with additive error $\eps$ using a total number of samples 
$$
T = \Theta(\frac{\sqrt{d}K\log K}{\eps^2}\log (KM/\delta)).
$$
}
\begin{proof}
Since $\|\mu_m\|, \|\Sigma_m\|$ are all bounded for all $m\in [M]$, we have that for each $k, k'\in [K]$ and $m\in [M]$, it holds that with probability $2/3$
$$
 |\hat{b}_k^{(m)}-b_k^{(m)}| = O(1/\sqrt{n}), 
 $$
 $$
 |\hat{H}_{k,k'}^{(m)} - H_{k,k'}^{(m)}| = O(\sqrt{n+d}/n).
 $$
Using the median of means trick as in Algorithm~\ref{alg:main}, we have that given $\frac{\sqrt{d}\log K }{\eps^2}\log (MK/\delta)$ iid samples for each arm $k$, it holds for all $m\in [M]$ that with probability $1-\delta$
$$
 |\hat{b}_k^{(m)}-b_k^{(m)}| = O(\eps/\sqrt{log K}), 
 $$
 $$
 |\hat{H}_{k,k'}^{(m)} - H_{k,k'}^{(m)}| = O(\eps^2/\log K),
 $$
where we need $\eps\ge \sqrt{\log K}/d^{1/4}$ for this to holds. The optimal reward for each component $m$ is 
$$
OPT^{(m)} = \E_{\x\sim N(\b^{(m)}, H^{(m)})}[\max_{k\in [K]} x_k],
$$
and by Proposition~\ref{prop:bH2OPT}, we can derive estimator  $\widehat{OPT}^{(m)}$ such that $|\widehat{OPT}^{(m)} - {OPT}^{(m)}|\le \eps$. Our final estimator satisfies
$$
|\widehat{OPT}-OPT| \le \sum_{i=1}^M \alpha_i|\widehat{OPT}^{(m)}-OPT^{(m)}| \le \eps,
$$
and uses a total of 
$$
T = \Theta(\frac{\sqrt{d}K\log K}{\eps^2}\log (KM/\delta))
$$
samples.
\end{proof}
\section{Minimax Lowerbound for Passive Algorithms}
\label{sec:minimax}
In this section, we prove the following proposition about the information theoretical lower bound for estimating the optimal reward, which is equivalent to Theorem~\ref{thm:lb}.
\begin{proposition}[Restatement of Theorem~\ref{thm:lb}]
Given $\frac{\sqrt{d}}{\eps}$ samples of each arm, there is no algorithm that can estimate the optimal reward with additive error $O(\sqrt{\eps\log K})$ with probability better than $2/3$.
\end{proposition}
\begin{proof}
We show our lower bound by upper bounding the total variational distance between the following two cases:
\begin{enumerate}
\item Draw $n$ independent samples $(\x_1,y_1),\ldots,(\x_n,y_n)$ where $\x_i\sim N(0,I), y_i\sim N(0,1)$. Repeat this procedure $K$ times.
\item First pick a uniformly random unit vector $v$ and set $b = \sqrt{\eps}$ with probability $1/\sqrt{K}$ and $b=0$ with probability $1-1/\sqrt{K}$, then draw $n$ independent samples $(\x_1,y_1),\ldots,(\x_n,y_n)$ where $\x_i\sim N(0,I), y_i = b v^T\x_i+\eta_i,$ where  $\eta_i\sim N(0, 1-b^2)$. Repeat this procedure $K$ times.
\end{enumerate}
The optimal reward of case $1$ is always $0$, while with the help of Fact~\ref{fact:glb}, it is easy to verify that the expected optimal reward of case $2$ is $\Omega(\sqrt{\eps\log K})$. We are going to prove that no algorithm can distinguish the two cases with probability more than $2/3$. Let $Q_n$ denote the joint distribution of $(\x_1,y_1),\ldots,(\x_n,y_n)$ in case $2$. Our goal is to bound the total variantion distance $D_{TV}(Q_n^{\otimes K},N(0,I)^{\otimes nK})$ which is smaller than $\frac{\sqrt{\chi^2(Q_n^{\otimes K},N(0,I)^{\otimes nK})}}{2}$ by the properties of chi-square divergence.

In case 2, for a fixed $v$ and $b$, the conditional distribution $\x|y\sim N(ybv,I-b^2vv^T)$. Let $P_{y,v}$ denote such a conditional distribution. The chi-square divergence can be expressed as:
\begin{align*}
&1+ \chi^2(Q_n^{\otimes K},N(0,I)^{\otimes nK})\\
&= 
(\int_{\x_1,y_1}\ldots \int_{\x_n,y_n}\frac{\Big(\frac{1}{\sqrt{K}}\int_{v\in \mathcal{S}^d} \prod_{i=1}^n  P_{y_i,v}(\x_i)G(y_i)dv+(1-\frac{1}{\sqrt{K}})\prod_{i=1}^n G(\x_i)G(y_i)\Big)^2}{\prod_{i=1}^n G(\x_i)G(y_i)}\\
&d\x_1dy_1\ldots d\x_ndy_n)^K\\
& = (\frac{1}{K}\int_{\x_1,y_1}\ldots \int_{\x_n,y_n}\int_{v\in \mathcal{S}^d} \int_{v'\in \mathcal{S}^d} \prod_{i=1}^n \frac{   P_{y_i,v}(\x_i)  P_{y_i,v'}(\x_i)G(y_i)}{G(\x_i)} dv dv' d\x_1dy_1\ldots d\x_ndy_n\\
& +(1-\frac{1}{K}))^K\\
& = \Big(\frac{1}{K}\int_{v\in \mathcal{S}^d} \int_{v'\in \mathcal{S}^d} \Big(\int_y \int_{\x}\frac{   P_{y,v}(\x)  P_{y,v'}(\x)G(y)}{G(\x)} d\x dy\Big)^n dv dv'+(1-\frac{1}{K})\Big)^K
\end{align*}
By the proof of Proposition 2 in~\cite{kong2018estimating}, we have $\int_{v\in \mathcal{S}^d} \int_{v'\in \mathcal{S}^d} \Big(\int_y \int_{\x}\frac{   P_{y,v}(\x)  P_{y,v'}(\x)G(y)}{G(\x)} d\x dy\Big)^n dv dv'\le 2$. Hence the above equation is bounded by $(1+\frac{1}{K})^K\le e$, and the total variation distance satisfies $D_{TV}(Q_n^{\otimes K},N(0,I)^{\otimes nK})\le 0.65$.
\end{proof}
\section{Minimax Lowerbound for Adaptive Algorithms}

This section is dedicated for the proof of Theorem~\ref{thm:lb-adaptive}. We restate Theorem~\ref{thm:lb-adaptive} as follows:

\smallskip
\vspace{.2cm}\noindent \textbf{Theorem~\ref{thm:lb-adaptive}.} \emph{In the known covariance setting, there exists a constant $C$ such that no algorithm can estimate the optimal reward with additive error $\eps$ with probability $2/3$ within 
$$T = C\frac{\sqrt{d}K}{\eps^{2}\log(dK)^{3/2}}$$
rounds.
}

We begin with some definitions of the notations to facilitate the proof.
\subsection{Notation}
Assuming we are in the contextual mult-armed bandit setting where each context $\x_i$ is drawn from $N(0,I_d)$, and a bandit is defined by the set of $K$ coefficient vectors $(\beta_1,\ldots,\beta_K)$. Given a policy $\pi$, and a bandit problem $\nu$, let $(\x_1,a_1,r_1, \ldots, \x_T,a_T,r_T)$ denote the context, action, reward trajectory induced by the policy $\pi$ and bandit $\nu$ with arms' coefficients $(\beta_1,\ldots,\beta_K)$, whose distribution is $\PB_\nu$, and let $\PB_{\nu'}$ be the distribution of the trajectory of problem $\nu'$ with arms' coefficients $(\beta'_1,\ldots,\beta'_K)$. 

For a fixed trajectory $(\x_1,a_1,r_1,\ldots, \x_T, a_T, r_T)$, let $T_a = \sum_{t=1}^T \bm{1}\{a_t=a\}$, $X_a\in R^{T_a\times d}$ consists of the $\x_t$'s where $a_t=a$, and $\r_a\in R^{T_a}$ consists of the $r_i$'s where $a_t=a$. Further, let $\x_{a,i}, i\in [T_a]$ be the columns of $X_a^\top$ and $r_{a,i}, i \in [T_a]$ be the elements of $\r_a$. Given $\x_1,\ldots, \x_T$, Let $S_i$, $i\in [\sum_{j=1}^{s} \binom{T}{s}]$ be all the subset of size at most $s$ of $\x_1,\ldots, \x_n$, and $W_i$ be the matrix whose rows are the elements of $S_i$.

Finally, we define $a^* = \argmin_{a}\E[T_a]$.
\subsection{Proof}
\textbf{Intuition:} One classical approach to prove regret lower bound in the stochastic bandit (non-contextural setting) is, for a given algorithm, to construct two bandit problem there is different in a single arm and bound the KL-divergence between the trajectories generated by the algorithm (see, e.g. Chapter 15 of~\cite{lattimore2018bandit}). Let $P_a, P_a'$ be the distribution of the reward of arm $a$ in the two problems. There is beautiful divergence decomposition result (Lemma 15.1 in~\cite{lattimore2018bandit}) which decompose the KL divergence between the trajectories as $\sum_{a=1}^{K} \E[T_a]D_{\text{KL}}(P_a, P_a')$. In our contextual bandit setting, roughly speaking, there is a similar decomposition, but instead of $\E[T_a]$, the KL divergence is roughly $\sum_{a=1}^K\E[T_a^2/d]$. Since it is possible to make $\E[T_a^2] = T^2/K$ for all $a$, which means that $T=\sqrt{Kd}$ suffices to make the KL-divergence greater than constant. Basically, the algorithm that randomly picks an arm and keeps pulling it for $T$ rounds is going to break the KL-divergence with bandit instances constructed this way.

However, it is clear that this algorithm is not going to succeed with probability more than $1/K$, and the total variation distance between the trajectories must be small. In order to get around with this issue with bounding KL-divergence, instead of focusing $\E[T_a]$ or $\E[T_a^2]$, we look at the probability that $T_a$ is greater than $\sqrt{d}$. Roughly speaking, there must be an arm $a$ such that $\Pr(T_a\ge \sqrt{d})$ is small (Fact~\ref{fact:Ta-markov}), and for these cases, we will bound the total variation just by its probability. While for the part where $T_a\le \sqrt{d}$, we will bound the KL divergence (Lemma~\ref{lem:KL-bound}) on the part and Pinsker inequality to obtain a total variation bound.

Our proof proceeds as follows. Given any adaptive algorithm that play the bandit game for $T$ rounds and output an estimate of $OPT$, we are going to find two bandit problem where the trajectory generated by the algorithm is indistinguishable in the two cases, while the $OPT$ in the two cases are very different. The following classical fact shows that as long as the trajectories is similar in the two cases, the output of the algorithm is going to be similar as well.
\begin{fact}\label{fact:trajectory-output}
Given any algorithm $A$ that interact with bandit and output a quantity $\widehat{OPT}$, let $\PB_{\nu}$, $\PB_{\nu'}$ be the distribution of the trajectory of $A$ interacting with $\nu, \nu'$, and $\QB_{\nu}, \QB_{\nu'}$ be the distribution of the output $\widehat{OPT}$ under $\nu$ and $\nu'$. If $D_{\text{TV}}(\PB_{\nu}, \PB_{\nu'})\le \delta$, then $D_{\text{TV}}(\QB_{\nu}, \QB_{\nu'})\le \delta$.
\end{fact}
Given this fact, what we need is to find the two bandit problems, such that $|{OPT}_\nu-{OPT}_{\nu'}|=\Theta(\eps)$, and the $D_{\text{TV}}(\PB_{\nu}, \PB_{\nu'})\le 1/3$. With a coupling argument, it is easy to see that the algorithm much incur $\Theta(\eps)$ error with probability $2/3$ in one of the two cases. The following lemma asserts the existence of such two bandit problems.
\begin{lemma}[Main lemma for the lower bound in the adaptive setting]\label{lem:tv-trajectory}
For any policy $\pi$, there exists two $K$-arm bandit $\nu$ and $\nu'$ such that $|\E_{\x\sim N(0,I_d)}[\max_i \beta_i\x]-\E_{\x\sim N(0,I_d)}[\max_i \beta'_i\x]|\ge \eps$, and with $T=\frac{C\sqrt{d}K}{\eps^2(\log dK)^{3/2}}$ rounds for a constant $C$, the total variance distance between the trajectories satisfies $D_{\text{TV}}(\PB_\nu, \PB_{\nu'})\le 1/3$.
\end{lemma}
Our main theorem of this section, Theorem~\ref{thm:lb-adaptive}, is immediately implied by Lemma~\ref{lem:tv-trajectory} and Fact~\ref{fact:trajectory-output}. We prove Lemma~\ref{lem:tv-trajectory} in the remainder of this section. 
\begin{fact}[Matrix concentration]\label{fact:gram-con}
Given $\x_1,\ldots, \x_T$ independently drawn from $N(0,I_d)$, let $S_i$, $i\in [\sum_{j=1}^{s} \binom{T}{s}]$ be all the subset of size at most $s$ of $\x_1,\ldots, \x_n$, and $W_i$ be the matrix whose rows are the elements of $S_i$. 
\begin{align*}
\Pr(\max_i\|I - W_i W_i^\top/d\|\ge  C \frac{\sqrt{s}}{\sqrt{d}}+\max(\frac{t}{\sqrt{d}}, \frac{t^2}{d}))\le \exp(s(1+\log (T/s))-ct^2)
\end{align*}
\end{fact}
\begin{proof}
The proof follows from Remark 5.59 of~\cite{vershynin2010introduction} and a union bound.
\end{proof}
The following fact shows that there exists an arm, such that with good probability, it does not get pulled by more than $O(T/K)$ times .
\begin{fact}\label{fact:Ta-markov}
Recall that $a^* = \argmin_{a}\E[T_a]$. Then
$$
\Pr(T_{a^*}\le \frac{1}{\delta}\frac{T}{K}) \ge 1-\delta.
$$
\end{fact}
\begin{proof}
Since $\E[\sum_{a=1}^{K} T_a]=T$, we have $\E[T_{a^*}]\le T/K$, and the claim then follows from Markov's inequality.
\end{proof}

We define the two instance $\nu, \nu'$ as follows,
\begin{definition}
We define $\nu$ to be the bandit problem with coefficient vectors $\beta_i=0$ for all $i\in [m]$ and the noise for each arm follows from $N(0,1)$, and $\nu'$ to be the same as $\nu$ except that with $\beta_{a^*}\sim N(0,\eps^2 I_d/d)$ and the noise of the arm $a^*$ is drawn from $N(0,1-\eps^2)$. 
\end{definition}

The following lemma shows that $T_{a^*}$ is small, and the context $\x_i$'s are ``typical'' with good probability.
\begin{lemma}[Good set]\label{lem:lb-goodset}
Define the set $E$ to be a set of the trajectories such that, for a constant $c$, for all $i \in[T]$, 
$r_i\le c\sqrt{\log T}$ and $|\x_i^\top\x_i/d-1|\le c\frac{\sqrt{\log T}}{\sqrt{d}}$, and for all $i \in [\sum_{j=1}^s\binom{T}{s}]$, $\|I-W_iW_i^\top/d\|\le c\frac{\sqrt{s\log T}}{\sqrt{d}}$, $T_{a^*}\le s$, where $s = c\frac{T}{K}$. Then there exists a constant $c$, such that $\PB_\nu(E)\ge 99/100$.
\end{lemma}

\begin{proof}
By Fact~\ref{fact:Ta-markov}, we can find a constant $c_1$ such that $\PB_\nu(T_{a^*}\le s = c_1 T/K)\le 1-1/1000$. Notice that under $\nu$, each reward $r_i\sim N(0,1)$, and by Fact~\ref{fact:gub} we can find a constant $c_2$ such that $r_i \le c_2\sqrt{T}$ for all $i\in T$ with probability $1-1/1000$. By Fact~\ref{fact:gram-con}, we can find a constant $c_3$ such that for all $i \in [\sum_{j=1}^s\binom{T}{s}]$, $\|I-W_iW_i^\top/d\|\le c\frac{\sqrt{s\log T}}{\sqrt{d}}$ with probability $1-1/1000$. Finally, by Fact~\ref{fact:gram-con} again, we can find a constant $c_4$ such that $|\x_i^\top\x_i/d-1|\le c_4\frac{\sqrt{\log T}}{\sqrt{d}}$.
Taking a union of the three events and $c = \max(c_1,c_2,c_3, c_4)$ completes the proof.
\end{proof}


Finally, the following lemma bound the KL-divergence on the good set, which will be used to bound the total variation with Pinsker inequality.
\begin{lemma}\label{lem:KL-bound}
$$
-\int_{E}d\PB_{\nu}\log \frac{d\PB_{\nu'}}{d\PB_{\nu}} \le 1/50.
$$
\end{lemma}
We leave the proof of this lemma to the end of this section, and prove the main lemma of this section.
\begin{proof}[Proof of Lemma~\ref{lem:tv-trajectory}]
The total variation distance between $\PB_\nu$ and $\PB_{\nu'}$,
\begin{align*}
D_{\text{TV}}(\PB_{\nu}, \PB_{\nu'}) \le& \frac{1}{2}(\int_{E^{c}} d\PB_{\nu'} + \int_{E^c} d\PB_{\nu} + \int_{E} |d\PB_{\nu'}-d\PB_{\nu}|)\\
\le& \int_{E^{c}} d\PB_{\nu} + \int_{E} |d\PB_{\nu'}-d\PB_{\nu}|\\
\le& \int_{E^{c}} d\PB_{\nu} + \sqrt{2} \sqrt{-\int_{E}d\PB_{\nu}\log \frac{d\PB_{\nu'}}{d\PB_{\nu}} +\int_Ed\PB_{\nu'}-\int_E d\PB_{\nu}}\\
\le& 1/100+\sqrt{2}\sqrt{\int_{E}d\PB_{\nu}\log \frac{d\PB_{\nu'}}{d\PB_{\nu}}+1/100}\\
\le& 1/3,
\end{align*}
where we applied Pinsker's inequality (Fact~\ref{fact:pinsker}) in the third last inequality, applied Lemma~\ref{lem:lb-goodset} in the second inequality, and applied Lemma~\ref{lem:KL-bound} in the last inequality.

\end{proof}

\begin{proof}
The density of $\PB_\nu$ can be expressed as
$$
p_{\nu}(\x_1,a_1,r_1, \ldots, \x_T,a_T,r_T) = \prod_{t=1}^T \pi_t(a_t|\x_1,a_1,y_1,\ldots,x_{t-1}, a_{t-1}, y_{t-1}, \x_t)p(y_t|\x_t, a_t),
$$
where $p(r_t|\x_t, a_t)$ is the density of reward $r_i$ on context $\x_i$ and arm $a_t$ in model $\nu$. The density of $\PB_{\nu'}$ is identical except that $p(r_t|\x_t, a_t)$ is replaced by $p'(r_t|\x_t, a_t)$. Then
$$
\log(\frac{d\PB_{\nu'}}{d\PB_{\nu}}(\x_1,a_1,r_1, \ldots, \x_T,a_T,r_T) = \sum_{t=1}^T \log \frac{p'(r_t|\x_t,a_t)}{p(r_t|\x_t,a_t)},
$$
and 
$$
-\int_Ed\PB_{\nu}\log(\frac{d\PB_{\nu'}}{d\PB_{\nu}}) = \sum_{t=1}^{T}\int_E \log \frac{p'(r_t|\x_t,a_t)}{p(r_t|\x_t,a_t)}d\PB_{\nu} .
$$
Under this setting, we have
\begin{align}
-\int_E d\PB_{\nu}\log(\frac{d\PB_{\nu'}}{d\PB_{\nu}}) = -\int_E d\PB_{\nu}\log \E_{\beta_{a^*}} \prod_{t=1}^T \frac{p'(r_t|\x_t,a_t)}{p(r_t|\x_t,a_t)} \nonumber \\
= -\int_E d\PB_{\nu}\log \E_{\beta_{a^*}}[\prod_{t=1}^T\bm{1}\{a_t=a^*\}\frac{1}{\sqrt{1-\eps^2}}\exp(-\frac{(r_t-\x_t^T\beta_{a^*})^2}{2(1-\eps^2)}+\frac{r_t^2}{2})]\label{eqn:KL-one-1}.
\end{align}
We compute the closed form expression of the expectation term as follows,
\begin{align*}
&\E_{\beta_{a^*}}[\prod_{t=1}^T\bm{1}\{a_t=a^*\}\frac{1}{\sqrt{1-\eps^2}}\exp(-\frac{(r_t-\x_t^T\beta_{a^*})^2}{2(1-\eps^2)}+\frac{r_t^2}{2})]\\
=& (2\pi)^{-d/2}(d/\eps^2)^{d/2}\int_{R^d} {(1-\eps^2)^{-T_{a^*}/2}}\exp(-\left(\beta_{a^*}^\top \frac{X_{a^*}^\top X_{a^*}}{2(1-\eps^2)}\beta_{a^*}-\frac{(X_{a^*}^\top \r_{a^*})^\top}{1-\eps^2}\beta_{a^*}+\frac{\eps^2}{2(1-\eps^2)}\r_{a^*}^\top \r_{a^*}\right)\\
&-\beta_{a^*}^\top\frac{dI_d}{2\eps^2}\beta_{a^*}) d\beta_{a^*}\\
 =& (2\pi)^{-d/2}(d/\eps^2)^{d/2}{(1-\eps^2)^{-T_{a^*}/2}}\exp(-\frac{\eps^2}{2(1-\eps^2)}\r_{a^*}^\top \r_{a^*}) \int_{R^d}\exp(-\frac{1}{2}\beta_{a^*}^\top A\beta_{a^*}+B\beta_{a^*}) d\beta_{a^*}
\end{align*}
where $A = (\frac{dI_d}{\eps^2}+\frac{X_{a^*}^\top X_{a^*}}{(1-\eps^2)})$, $B = \frac{X_{a^*}^\top \r_{a^*}}{(1-\eps^2)}$. We can now apply the Gauss integral property and get that the last line equals
\begin{align*}
 =& (2\pi)^{-d/2}(d/\eps^2)^{d/2}{(1-\eps^2)^{-T_{a^*}/2}}\exp(-\frac{\eps^2}{2(1-\eps^2)}\r_{a^*}^\top \r_{a^*}) \sqrt{\frac{(2\pi)^d}{\det(A)}}\exp(\frac{1}{2}B^\top A^{-1}B)\\
=& (d/\eps^2)^{d/2}{(1-\eps^2)^{-T_{a^*}/2}}\det(A)^{-1/2} \exp(-\frac{\eps^2}{2(1-\eps^2)}\r_{a^*}^\top \r_{a^*}) \exp(\frac{1}{2}B^\top A^{-1}B).
\end{align*}
Plugging in the above formula to Equation~\ref{eqn:KL-one-1}, we have that Equation~\ref{eqn:KL-one-1} equals
\begin{align}
=& \frac{1}{2}\left(\int_Ed\PB_\nu\log\det(A) -d\log (d/\eps^2) +\int_E d\PB_{\nu}T_{a^*}\log (1-\eps^2)\right)\nonumber\\
&+\frac{1}{2}\left(\frac{\eps^2}{(1-\eps^2)}\int_Ed\PB_\nu\r_{a^*}^\top \r_{a^*} - \int_Ed\PB_\nu B^\top A^{-1}B\right).\label{eqn:KL-one-2}
\end{align}
Let $\lambda_1\ge \lambda_2\ge\ldots\lambda_{T_{a^*}}$ be the eigenvalues of matrix $X_{a^*}^\top X_{a^*}$. Then first term can be written as 
\begin{align}
&\frac{1}{2}\left(\int_Ed\PB_\nu\log\det(A) -d\log (d/\eps^2) +\int_E d\PB_{\nu}T_{a^*}\log (1-\eps^2)\right)\nonumber \\
=& \frac{1}{2}\int_Ed\PB_\nu\sum_{i=1}^{T_{a^*}} \log(1+\eps^2(\lambda_i/d-1)) \le \frac{\eps^2}{2}\int_Ed\PB_\nu\sum_{i=1}^{T_{a^*}}(\lambda_i/d-1)\nonumber \\
=& \frac{\eps^2}{2}\int_Ed\PB_{\nu}\sum_{i=1}^{T_{a^*}}(\x_{a,i}^\top\x_{a,i}/d-1) \le c\frac{\eps^2 T\sqrt{\log T}}{K\sqrt{d}}\label{eqn:KL-1}
\end{align}
for a constant $c$, where in the third last inequality we used the fact that $\log(1+x)<x$, in the second last inequality we used the fact that $\sum_{i=1}^{T_{a^*}}\lambda_i = \sum_{i=1}^{T_{a^*}}\x_{a,i}^\top \x_{a,i}$ and in the last inequality used Lemma~\ref{lem:lb-goodset} that under set $E$, $T_a\le O(\frac{T}{K})$, $\x_{a,i}^\top \x_{a,i}/d-1\le O(\frac{\sqrt{T}}{\sqrt{d}})$. 

For the second termin Equation~\ref{eqn:KL-one-2}, notice that the eigenvalues of $\frac{\eps^2}{1-\eps^2}I - \frac{X_{a^*}}{(1-\eps^2)} A^{-1} \frac{X_{a^*}^\top}{(1-\eps^2)}$ are
\begin{align*}
\frac{\eps^2}{1-\eps^2}-\frac{\lambda_i/(1-\eps^2)^2}{d/\eps^2+\lambda_i/(1-\eps^2)} = 
\frac{\eps^2(1-\lambda_i/d)}{1-\eps^2(1-\lambda_i/d)}=\sum_{k=1}^\infty (\eps^2(1-\lambda_i/d))^k
\end{align*} and hence
$$
\frac{\eps^2}{1-\eps^2}I - \frac{X_{a^*}}{(1-\eps^2)} A^{-1} \frac{X_{a^*}^\top}{(1-\eps^2)} = \sum_{k=1}^\infty \left(\eps^2(I_d-X_{a^*}X_{a^*}^\top/d)\right)^k.
$$
Plugging in the expression into
the second term of Equation~\ref{eqn:KL-one-2}, the term becomes
\begin{align*}
\frac{1}{2}\sum_{k=1}^\infty \eps^{2k}\int_E d\PB_{\nu}\r_{a^*}^\top \left(I_d-X_{a^*}X_{a^*}^\top/d\right)^k\r_{a^*}
\end{align*}
For $k=1$, we have 
\begin{align}
&2^{-1}\eps^2\int_E d\PB_{\nu}\r_{a^*}^\top (I_d-X_{a^*}X_{a^*}^\top/d)\r_{a^*}\nonumber \\
=& 2^{-1}\eps^2\Big(\int_Ed\PB_\nu\sum_{i=1}^{T_{a^*}} r_{a^*,i}^2(1-\x_{a^*,i}^\top\x_{a^*,i}/d)+ 2\int_E d\PB_{\nu} \sum_{i< j}r_{a^*,i}r_{a^*,j}\x_{a^*,i}^\top\x_{a^*,j}/d\Big)\nonumber\\
\le& c\eps^2(\frac{T(\log T)^{3/2}}{K\sqrt{d}}+\frac{T^2\log T}{K^2d})+\eps^2/100,\label{eqn:KL-2}
\end{align}
for a constant $c$, where the last equality holds by simply expanding the maxtrix multiplication, and the last inequality holds due to Lemma~\ref{lem:lb-goodset} and Lemma~\ref{lem:lb-cross-term}.

For the remaining terms with $k\ge 2$, 
\begin{align}
&\frac{1}{2}\int_E\PB_\nu \r_a^\top \left(\sum_{k=2}^\infty \eps^{2k}(I_d-X_aX_a^\top/d)^k\right) \r_a \le \int_E \PB_\nu \r_a^\top\r_a \|\sum_{k=2}^\infty \eps^{2k}(I_d-X_aX_a^\top/d)^k\| \nonumber\\
\le& c\eps^4\frac{T^2(\log T)^2}{K^2d},\label{eqn:KL-3}
\end{align}
for some constant $c$, where the last inequality holds due to the fact that $\r_a^\top \r_a \le O(\frac{T\log T}{K})$, $\|\eps^{4}(I_d-X_aX_a^\top/d)^2\|\le O(\frac{T\log T}{Kd})$ by Lemma~\ref{lem:lb-goodset}.
Combing Equation~\ref{eqn:KL-1}, ~\ref{eqn:KL-2}, ~\ref{eqn:KL-3}, we have
\begin{align*}
&-\int_E\PB_\nu \log \frac{d\PB_{\nu'}}{d\PB_{\nu}} \le O(\eps^4\frac{T^2(\log T)^2}{K^2d}+\eps^2\frac{T(\log T)^{3/2}}{K\sqrt{d}})+\eps^2/100.
\end{align*}
Since $\eps\le 1$, we can find a constant $C$ such that setting $T = \frac{C K\sqrt{d}}{\eps^2\log(Kd)^{3/2}}$ gives
$$
-\int_E\PB_\nu \log \frac{d\PB_{\nu'}}{d\PB_{\nu}} \le 1/50.
$$
This conclude the proof.

\begin{lemma}\label{lem:lb-cross-term}
$$
\int_E d\PB_{\nu} \sum_{i< j}r_{a^*,i}r_{a^*,j}\x_{a^*,i}^\top\x_{a^*,j}/d \ge -C\frac{T^2 \log T }{K^2d} - 1/100.
$$
for a positive constant $C$.
\end{lemma}
Notice that by martingale stopping theorem, 
\begin{align*}
&\int_{E^c} d\PB_{\nu} \sum_{ 1\le i < j\le s}r_{a^*,i}r_{a^*,j}\x_{a^*,i}^\top\x_{a^*,j} /d + \int_{E} d\PB_{\nu} \sum_{ 1\le i < j\le s}r_{a^*,i}r_{a^*,j}\x_{a^*,i}^\top\x_{a^*,j} /d \\
 &= \E_{\PB_{\nu}}[\sum_{1\le i< j\le s}r_{a^*,i}r_{a^*,j}\x_{a^*,i}^\top\x_{a^*,j}/d] = 0.
\end{align*}
Hence we are going to upper bound 
\begin{align}
\int_{E^c} d\PB_{\nu} \sum_{ 1\le i < j\le s}r_{a^*,i}r_{a^*,j}\x_{a^*,i}^\top\x_{a^*,j} /d,\label{eqn:lb-Ec}
\end{align}
where $E^c$ is the complement of set $E$, and this is going to imply the lower bound in the lemma. Equation~\ref{eqn:lb-Ec} is bounded by,
\begin{align*}
&\int_{E^c} d\PB_{\nu} \sum_{1\le i < j \le s}r_{a^*,i}r_{a^*,j}\x_{a^*,i}^\top\x_{a^*,j}/d \le \int_{E^c} d\PB_{\nu} (\sum_{1\le i < j \le s}r_{a^*,i}r_{a^*,j}\x_{a^*,i}^\top\x_{a^*,j}/d)^2+\int_{E^c}d\PB_{\nu}\\
&\le \E_{\PB_{\nu}}[(\sum_{1\le i < j \le s}r_{a^*,i} r_{a^*,j}\x_{a^*,i}^\top\x_{a^*,j}/d)^2] + 1/100,\end{align*}
and we have 
\begin{align}
&\E_{\PB_{\nu}}[(\sum_{1\le i < j \le s}r_{a^*,i} r_{a^*,j}\x_{a^*,i}^\top\x_{a^*,j}/d)^2]\nonumber\\
=& \E_{\PB_{\nu}}[\sum_{i = 1}^s\Big((\sum_{j=1}^{i-1} r_{a^*,j} \x_{a^*,j}^\top)r_{a^*,i}\x_{a^*,i}/d\Big)^2]\label{eqn:lb-Exp-1}\\
\le& \E_{\PB_{\nu}}[\sum_{i = 1}^s\Big((\sum_{j=1}^{i-1} r_{a^*,j} \x_{a^*,j}^\top)\x_{a^*,i}/d\Big)^2]\label{eqn:lb-Exp-2}\end{align},
where Equation~\ref{eqn:lb-Exp-1} holds due to the fact that for $i<j, i'<j'$, $\E_{\PB_{\nu}}[r_{a^*,i}r_{a^*,j}r_{a^*,i'}r_{a^*,j'}\x_{a^*,i}^\top\x_{a^*,j} \x_{a^*,i'}^\top\x_{a^*,j'}^\top]=0$ unless $j=j'$., Formula~\ref{eqn:lb-Exp-2} holds due to $\E_{\nu}[r_{a^*,i}^2] = 1$ and $r_{a^*,i}$ is independent of $r_{a^*,j}, \x_{a^*, j}$ where $j<i$.

Notice that for a single term $\Big((\sum_{j=1}^{i-1} r_{a^*,j} \x_{a^*,j}^\top)\x_{a^*,i}/d\Big)^2$ in Formula~\ref{eqn:lb-Exp-2}, if we fixed $r_{a^*, j}, \x_{a^*,j }$ for all $j<i$, the algorithm must pick $\x_{a^*,i}$ from the remaining contexts which is generated independent of $(\sum_{j=1}^{i-1} r_{a^*,j} \x_{a^*,j}^\top)$. Hence  
$$
\E_{\nu}[\Big((\sum_{j=1}^{i-1} r_{a^*,j} \x_{a^*,j}^\top)\x_{a^*,i}/d\Big)^2|\{r_{a^*, j}, \x_{a^*,j} \}_{j<i}]\le \E_{\z}[\max z_i] = O(\log T\|\sum_{j=1}^{i-1} r_{a^*,j} \x_{a^*,j}^\top\|^2/d^2),
$$
where $z_i\sim N(0,  \|\sum_{j=1}^{i-1} r_{a^*,j} \x_{a^*,j}^\top\|^2/d^2)$ and we have that Formula~\ref{eqn:lb-Exp-2} is bounded by:
\begin{align}
=& O(\log T \E_{\PB_{\nu}}[\sum_{i = 1}^s\|\sum_{j=1}^{i-1} r_{a^*,j} \x_{a^*,j}^\top\|^2/d^2])\nonumber\\
=&  O(\log T\E_{\PB_{\nu}}[\sum_{i = 1}^s \sum_{j=1}^{i-1} r_{a^*,j}^2 \x_{a^*,j}^\top\x_{a^*,j}/d^2])\label{eqn:lb-Exp-4}\\
=& O(\log T s^2/d) = O(\frac{T^2 \log T }{K^2d})\label{eqn:lb-Exp-5},
\end{align}
where Equation~\ref{eqn:lb-Exp-4} holds since for $i\le j$, $\E_{\nu}[r_{a^*,i}r_{a^*,j}\x_{a^*,i}^\top \x_{a^*,j}] = 0$, and Equation~\ref{eqn:lb-Exp-5} holds due to Corollary~\ref{cor:con-x-norm}. 

Hence we have that Equation~\ref{eqn:lb-Ec} is bounded by $O(\frac{T^2 \log T }{K^2d})+1/100$, and hence the lemma holds.
\end{proof}

The following statement is a standard statement of the concentration of the norm.
\begin{cor}[Concentration of the norm]\label{cor:con-x-norm}
Let $\x_1,\ldots, \x_T$ be independently drawn from $N(0,I_d)$. Then $\E[\max_i \x_i^\top \x_i/d]\le 1+O(\frac{\sqrt{\log(Td)}}{\sqrt{d}})$.
\end{cor}
\begin{proof}
By Fact~\ref{fact:gram-con}, we have 
\begin{align*}
\Pr(\max_i |\x_i^\top \x_i/d-1|\ge  \frac{C}{\sqrt{d}}+\max ( \frac{t}{\sqrt{d}}, \frac{t^2}{d}) )\le \exp((1+\log T)-ct^2)\\
\implies \Pr(\max_i |\x_i^\top \x_i/d-1|\ge  \frac{C}{\sqrt{d}}+t) )\le \exp((1+\log T)-cd\min(t^2,t))
\end{align*}
By the fact that for any random variable $X$, $\E[X]\le t+\int_{x=t}^{\infty} \Pr(X>x)dx$, we have 
\begin{align*}
&\E[\max_i |\x_i^\top \x_i/d-1|] \le \frac{C+\sqrt{\log (Td)/c}}{\sqrt{d}}+ \int_{t = \sqrt{\log T}/\sqrt{cd}}^\infty \exp((1+\log (Td))-cd\min(t^2,t)) \\
&\le O(\frac{\sqrt{\log(Td)}}{\sqrt{d}}).
\end{align*}
\end{proof}

\section{Auxiliary Lemmas}
\begin{fact}[Pinsker's inequality for arbitrary measure]\label{fact:pinsker} Let $P, Q$ be two positive measure such that $\int dP \le 1, \int dQ \le 1$. Then 
$$
\frac{1}{2}\int |dP-dQ| \le 2^{-1/2}\sqrt{-\int dP\log\frac{dQ}{dP}+\int dQ - \int dP}.
$$
\end{fact}
\begin{proof}
The proof is classic, and we follows the proof of Lemma 2.5 of the book~\cite{tsybakov2008nonparametric}. Notice that the difference between this version of Pinsker's inequality and the classic one is that $\int P d\mu$ and $\int Q d\mu$ do not need to be $1$, and the proof follows until the last part (first paragraph of page 89 on~\cite{tsybakov2008nonparametric}) where we have
\begin{align*}
\frac{1}{2}\int |dP-dQ| &\le \frac{1}{2}\sqrt{\int (\frac{4}{3}dQ+\frac{2}{3}dP)}\sqrt{\int (dP\log\frac{dP}{dQ}+dQ-dP) }\\
&\le 2^{-1/2}\sqrt{-\int dP\log\frac{dQ}{dP}+\int dQ - \int dP}.
\end{align*}
\end{proof}
\begin{fact}[Upper bound of the expectation of the maximum of Gaussians, see e.g. \cite{kamath2015bounds}]\label{fact:gub}
Given that $\x\sim N(0,\Sigma)$ where $\Sigma\in R^{m\times m}$ and $\Sigma_{i,i}\le \sigma^2$ for all $i=1,\ldots,m$, $\E[\max{|x_i|}]\le \sqrt{2}\sigma\sqrt{\log m}$
\end{fact}
\begin{fact}[Lower bound of the expectation of the maximum of Gaussians, see e.g. \cite{kamath2015bounds}]\label{fact:glb}
Given that $\x\sim N(0,I_m)$, $\E[\max{|x_i|}]\ge 0.23 \sqrt{\log m}$
\end{fact}
\begin{fact}[Median of means trick]\label{fact:boosting}
Given a randomized algorithm that, with probability $2/3$, output an estimate $\hat{x}$ such that $|\hat{x}-x|\le \eps$. If we independently execute the algorithm $t$ times, the median of the estimates satisfies $|\text{median}(\hat{x}_1,\ldots, \hat{x}_t)-x|\le \eps$ with probability at least $1-\exp(-t/48)$.
\end{fact}
\begin{proof}
Notice that if there is more than $t/2$ estimates that fall into the interval $[x-\eps,x+\eps]$, the median of the estimates must have error less than $\eps$. Hence, we only need to upper bound the probability that the there are less than $t/2$ estimates that fall into the interval $[x-\eps,x+\eps]$. Let $z_i$ be the indicator random variable of whether $\hat{x}_i$ fall into the interval $[x-\eps,x+\eps]$. By Chernoff bound (Fact~\ref{fact:chernoff}), we have
\begin{align*}
\Pr(\sum_{i=1}^t z_i\le t/2) = \Pr(\sum_{i=1}^t z_i\le (1-\frac{1}{4})\frac{2}{3}t) )
\le \exp(-\frac{t}{48})
\end{align*}
\end{proof}
\begin{fact}[Chernoff Bound]\label{fact:chernoff}
Suppose $X_1,\ldots, X_n$ are independent random variables taking values in $\{0, 1\}$ with $\mu = \E[\sum_{i=1}^nX_i]$. Then for any $\delta > 0$,
$$
\Pr(\sum_{i=1}^nX_i\le (1-\delta)\mu) \le \exp(-\frac{\delta^2\mu}{2})
$$
\end{fact}
\end{document}